\documentclass{article}

\usepackage{microtype}
\usepackage{graphicx}
\usepackage{subfigure}
\usepackage{wrapfig}
\usepackage{pgfplots}
\usepackage{multirow,tabularx}
\pgfplotsset{compat = newest}
\usetikzlibrary{positioning, arrows.meta}
\usepgfplotslibrary{fillbetween}
\usepackage{booktabs} 

\usepackage{hyperref}
\usepackage{listings}

\usepackage[accepted]{icml2024}

\usepackage{amsmath}
\usepackage{bbm}
\usepackage{amssymb}
\usepackage{mathtools}
\usepackage{amsthm}
\def\thickhline{\noalign{\hrule height.8pt}}
\usepackage{multicol}
\newcommand\indicfcn{{\mathbbm{1}}}
\usepackage[capitalize,noabbrev]{cleveref}

\theoremstyle{plain}
\newtheorem{theorem}{Theorem}[section]

\newtheorem{lemma}[theorem]{Lemma}

\theoremstyle{definition}
\newtheorem{definition}[theorem]{Definition}
\newtheorem{assumption}[theorem]{Assumption}
\theoremstyle{remark}

\usepackage[textsize=tiny]{todonotes}

\icmltitlerunning{Riemannian Preconditioned LoRA}

\begin{document}

\twocolumn[
\icmltitle{Riemannian Preconditioned LoRA for Fine-Tuning Foundation Models}



\icmlsetsymbol{equal}{*}

\begin{icmlauthorlist}
\icmlauthor{Fangzhao Zhang}{yyy}
\icmlauthor{Mert Pilanci}{yyy}
\end{icmlauthorlist}

\icmlaffiliation{yyy}{Department of Electrical Engineering, Stanford University}

\icmlcorrespondingauthor{Fangzhao Zhang}{zfzhao@stanford.edu}

\icmlkeywords{Machine Learning, ICML}

\vskip 0.3in
]



\printAffiliationsAndNotice{}  

\begin{abstract}
Low-Rank Adaptation (LoRA) emerges as a popular parameter-efficient fine-tuning (PEFT) method, which proposes to freeze pretrained model weights and update an additive low-rank trainable matrix. In this work, we study the enhancement of LoRA training by introducing an $r\times r$  preconditioner in each gradient step where $r$ is the LoRA rank. We theoretically verify that the proposed preconditioner stabilizes feature learning  with LoRA under infinite-width NN setting. Empirically, the implementation of this new preconditioner requires a small change to existing optimizer code and creates virtually minuscule storage and runtime overhead. Our experimental results with both large language models and text-to-image diffusion models show that with this new preconditioner, the convergence and reliability of SGD and AdamW can be significantly enhanced. 
Moreover, the training process becomes much more robust to hyperparameter choices such as  learning rate. The new preconditioner can be derived from a novel Riemannian metric in low-rank matrix field. Code can be accessed at \url{https://github.com/pilancilab/Riemannian_Preconditioned_LoRA}.
\end{abstract}

\section{Introduction}\label{intro}
With the expanding scale of neural network models in both vision and language domains, training a neural network from scratch to match the performance of existing large models has become almost infeasible. As a result, fine-tuning has emerged as a prevalent approach for downstream tasks. Traditional full parameter fine-tuning demands extensive storage, making it impractical for many applications. In contrast, recent advances in Parameter-Efficient Fine-Tuning (PEFT) methods offer a more storage-efficient solution while still delivering strong performance in downstream tasks. A widely-used PEFT method is Low Rank Adaptation, also known as LoRA \cite{hu2021lora}, which proposes to add low-rank matrices to existing model weights and only train these additive components. Simply speaking, for a pretrained model weight matrix $W$ of dimension $m$ by $n$, LoRA replaces $W$ with $W+BA$ where $B,A$ are trainable weight matrices of dimension $m$ by $r$ and $r$ by $n$ for some small rank $r$. In the fine-tuning procedure, $W$ is frozen and we are only optimizing over $A$'s and $B$'s. Compared to full fine-tuning, LoRA introduces fewer trainable parameters. Effectiveness of LoRA has been  empirically verified in different fields.  Here we note that optimizing LoRA parameters falls into optimizing over low-rank matrices which form a quotient manifold. This motivates the idea of enhancing LoRA training via tools from the field of Riemannian optimization.

Recent work such as LoRA+ \cite{hayou2024lora} has drawn attention to the optimization paradigm of LoRA and reveals that the learning rate of LoRA parameter $B$ should be set larger than that of $A$ to achieve stable feature learning  under the infinite-width NN setting, making the learning rate tuning a joint hyperparameter search. More specifically, denote learning rates for $A$ and $B$ by $\eta_A$ and $\eta_B$ respectively, LoRA+ proposes to use a heuristic $\eta_B/\eta_A=2^4$ in practice and tune only $\eta_A$.  In this work, we propose an improvement of LoRA training with the introduction of an $r\times r$ preconditioner in its optimization step, and we show that with this simple  preconditioner, LoRA learns stable features without the need of setting different learning rates for $A$ and $B$. As an illustration, we propose modifying the gradient updates for the LoRA parameters as follows
\begin{equation}
\begin{aligned}
A_{t+1}=A_t-\alpha (B_t^TB_t)^{-1}(\nabla_{A_t}\mathcal{L}),\\
B_{t+1}=B_t-\alpha(\nabla_{B_t}\mathcal{L})(A_tA_t^T)^{-1},
\end{aligned}
\end{equation}\label{formula}
 \begin{figure*}[ht!]
 \vspace{-0.2cm}
 \centering
\includegraphics[width=0.85\linewidth]{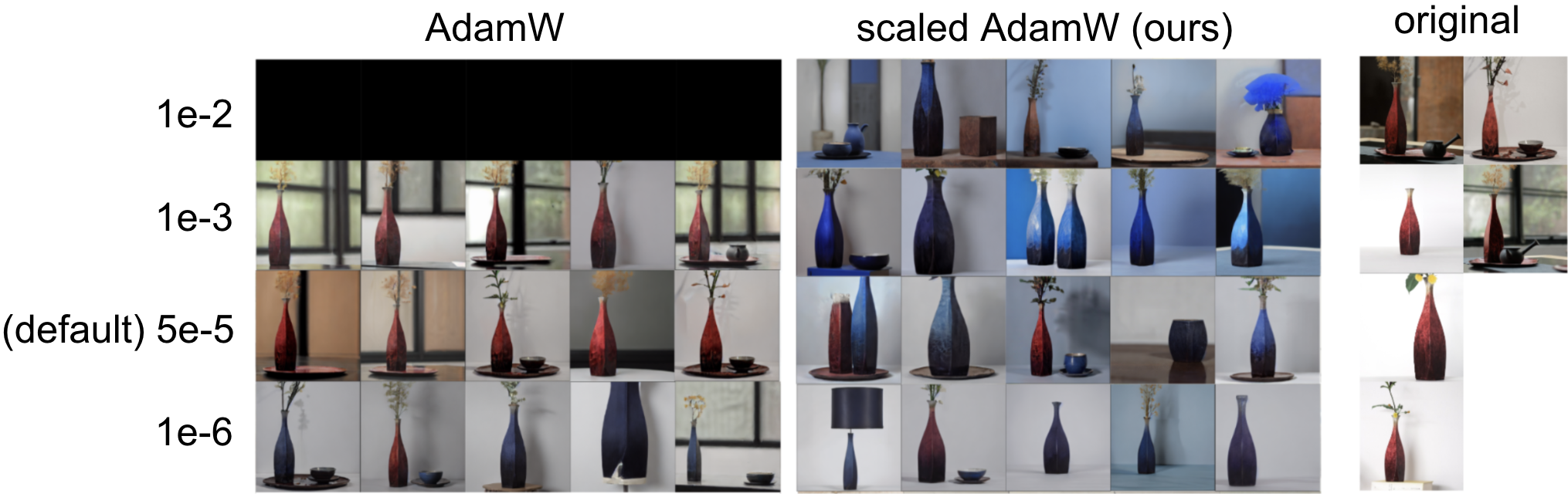}
\caption{
Generation results for prompt ``a blue $\langle V_{\text{vase}} \rangle$'' after fine-tuning on $6$ red vase images  of the Stable Diffusion V1.5 model. No black images are observed for our method (scaled AdamW)'s generation and AdamW generates only black images for large learning rates. Our method generates photos better capturing the prompt and is more robust to learning rate changes. See Section \ref{object_section} for experimental details. 
}\label{vase_main}
\end{figure*}
\vspace{-0.35cm}

where $(A_tA_t^T)^{-1}$ and $(B_t^TB_t)^{-1}$ are  the preconditioners we introduce and $\mathcal{L}$ is the training loss objective. This scaled gradient method (scaled GD) can be derived from a novel Riemannian metric studied in \cite{mishra2012riemannian} 
 which takes into consideration both the objective function and constraints and has been shown to have better convergence rate for  conventional low-rank matrix optimization problems involving matrix sensing and robust PCA \cite{candes2009robust}. Our work shows for the first time that this preconditioner enables LoRA to achieve stable feature learning with Adam optimizer, revealing the superiority of preconditioning in fine-tuning deep learning models. Empirically, we verify the validity of this new preconditioner for LoRA training via extensive fine-tuning tasks for both language models and text-to-image diffusion models. The experimental results show that the convergence of both SGD and AdamW is significantly enhanced by preconditioning (gradient scaling). Despite the accelerated convergence, the optimization procedure also becomes   more robust to learning rate changes with this new preconditioner. Figure \ref{vase_main} compares the generation results for diffusion model fine-tuning with AdamW optimizer with and without gradient scaling, from which it can be observed that  gradient scaling significantly improves the generation quality as well as training robustness against learning rate changes, see Algorithm \ref{algo:code} for the formal algorithm for scaled AdamW. Most importantly, our preconditioner is only of dimension $r$ by $r$ thus the storage overhead is negligible for small LoRA rank $r$. Also, unlike most second-order preconditioners based on Hessian inverses, inverting an $r$ by $r$ matrix for small $r$ introduces negligible runtime overhead and runs as fast as unpreconditioned optimizers. In practice, small values, e.g., $r=4$, are usually used as default for LoRA fine-tuning, see Figure \ref{small_GPU} for runtime comparison when fine-tuning GPT-2 with both scaled and unscaled optimizers.

To summarize, we study the application of  preconditioned gradient updates (\ref{formula}) in LoRA training. Theoretically, we show that the proposed preconditioner provides an elegant solution to learning rate choices for stable feature learning with LoRA under infinite-width NN setting. We also provide a convergence guarantee for fine-tuning a reparameterized model which is equivalent to a two-layer ReLU network via convexification. See Section \ref{stable_feature} and Section \ref{theory} respectively for details. Empirically, we apply this method for LoRA fine-tuning for both large language and  diffusion models and observe that scaled optimizers significantly improve the performance of unscaled optimizers. See Section \ref{simulation} for the experimental details. To the best of our knowledge, this is the first work to apply Riemannian optimization in designing preconditioners for fine-tuning large foundation models. This approach is particularly appropriate given the matrix factorization characteristics of the LoRA model.

Next, we first provide some intuition on our method in Section \ref{intuition}. We then show that LoRA training with the proposed preconditioners achieves stable feature learning in Section \ref{stable_feature}. We include an overview of basic Riemannian optimization concepts as well as the derivation of this preconditioner by introduction of a new Riemannian metric  in Section \ref{metric_sec}. We offer a pseudocode and present extensive experimental results in Section \ref{simulation}.  Finally we state our convergence results for related problems in Section \ref{theory} and review prior literature in Section \ref{prior_work}.

\section{Notation}
We adopt same notation convention as in LoRA+ \cite{hayou2024lora}. Specifically, for any sequence $s_n\in\mathbb{R}$ and any given $c_n\in\mathbb{R}^+$, we use $s_n=\mathcal{O}(c_n)$ and $s_n=\Omega(c_n)$ to represent $s_n<\kappa c_n$ and $s_n>\kappa c_n$ respectively for some $\kappa>0.$ We use $s_n=\Theta(c_n)$ when both $s_n=\mathcal{O}(c_n)$ and $s_n=\Omega(c_n)$. For vector and matrix sequences, the notation is applied entrywise. For a sequence that is a vector of random variables, convergence in second moment, i.e., $L_2$ convergence, is adopted.
\section{Theoretical Insights}\label{intuition}
We first offer an intuitive explanation for the effectiveness of the  preconditioner (\ref{formula}) before we move on to its stable learning properties and  review its rigorous derivation from a Riemannian metric formulation. Consider in the $t$-th iteration  pretrained model weight $W$ and its additive low-rank component $B_tA_t$. Let $X_t=W+B_tA_t$ denote the whole weight matrix and let $\mathcal{L}$ denote the loss function, i.e., $\mathcal{L}(A,B):=\mathcal{L}(W+BA)$. For the plain gradient descent method, when step size $\eta$ is small, the updated weight is approximately given by 
\begin{equation*}
\begin{aligned}
X_{t+1}&=W+B_{t+1}A_{t+1}\\
&\approx W+B_tA_t-\eta \nabla_{B_t} \mathcal{L}A_t-\eta B_t\nabla_{A_t} \mathcal{L} \\
&=X_t-\eta \nabla_{X_t} \mathcal{L}(A_t^TA_t)-\eta (B_tB_t^T)\nabla_{X_t} \mathcal{L},
\end{aligned}
\end{equation*}
where we ignore the second-order term in the second line since when $\eta$ is small, $\eta^2$ is negligible. The derivation from the second line to the third line comes from the simple fact $\nabla_{B_t} \mathcal{L}=(\nabla_{X_t} \mathcal{L})A_t^T$ and $\nabla_{A_t} \mathcal{L}=B_t^T(\nabla_{X_t} \mathcal{L}) $. Thus the LoRA update in gradient descent step is approximately  constrained to the column space of $B_t$ and the row space of $A_t$. If we scale $\nabla_{B_t} \mathcal{L}$ right by $(A_tA_t^T)^{-1}$ and scale $\nabla_{A_t} \mathcal{L}$ left by $(B_t^TB_t)^{-1}$, which is exactly the preconditioner (\ref{formula}), the scaled update becomes
\begin{equation*}
\begin{aligned}
\tilde X_{t+1}&\approx X_t-\eta (\nabla_{X_t} \mathcal{L})A_t^T(A_tA_t^T)^{-1}A_t\\& \qquad ~~-\eta B_t(B_t^TB_t)^{-1}B_t^T(\nabla_{X_t} \mathcal{L}) \\
&=X_t-\eta \mathrm{Proj}_{\text{row}(A_t)}(\nabla_{X_t} \mathcal{L})-\eta \mathrm{Proj}_{\text{col}(B_t)}(\nabla_{X_t} \mathcal{L})^T,
\end{aligned}
\end{equation*}
where the update is done according to projection of the full matrix gradient onto the row space of $A_t$ and the column space of $B_t$, which better approximates full fine-tuning compared to the unscaled gradient descent step. Therefore, our preconditioner (\ref{formula}) effectively serves as a gradient projector.
\section{Stable Feature Learning}\label{stable_feature}
Given the current trend of increasing model sizes, there has been a focus on analyzing the asymptotic training behavior of neural networks as the number of neurons approaches infinity \cite{schoenholz2017deep,hayou2019impact,yang2020scaling}. Under this infinite-width NN setting, we naturally expect that both the NN prediction $f^t(x)$ in the $t$-th iteration and the increment $\Delta f^t:=f^t(x)-f^{t-1}(x)$ to be of constant magnitude, which ensures that neither the NN predictions nor the increments explode or vanish as the NN size increases, thereby leading to stable training dynamics. We refer to this behavior as \emph{stable feature learning}, formally defined in Definition \ref{thm3_def} in the Appendix. The authors of LoRA+ \cite{hayou2024lora} observe that the learning rate of LoRA parameter $B$ should be set to be larger than that of $A$ for asymptotic stable feature learning. Their analysis mostly focuses on the order of magnitude of iterates and the conclusion that $\eta_B\gg \eta_A$ does not immediately offer practical guidance. Here we show that our preconditioned updates (\ref{formula}) introduce an elegant and practical solution to learning rate choices for stable feature learning in the infinite-width NN regime. Specifically, we will see that LoRA trained with Adam optimizer scaled by our preconditioner as in \eqref{formula} leads to stable feature learning when the same learning rate is used for training $A$ and $B$. In contrast, LoRA trained with the unpreconditioned Adam optimizer requires different learning rate settings for $A$ and $B$ to achieve stable feature learning. 

We first illustrate our key point via a simple toy example and we then proceed to our main theorem. 
 Consider the simple linear model 
\[f(x)=(W+ba^T)x,\]
where $W\in\mathbb{R}^{1\times n}$ is the pretrained model weight and $b\in\mathbb{R},a\in\mathbb{R}^n$ are trainable LoRA parameters. Consider the quadratic loss function $\mathcal{L}(a,b)=(f(x)-y)^2/2$ with some scalar label $y.$ We adopt Gaussian initialization $a_i\sim\mathcal{N}(0,\sigma_a^2),b\sim\mathcal{N}(0,\sigma_b^2)$. Conventionally, $ba^T$ is initialized at zero for LoRA, and we thus consider setting $\sigma_a^2=0,\sigma_b^2=\Theta(1)$.

\textbf{Analysis of unpreconditioned training. }Assume the model is trained with gradient descent with learning rate $\eta=\Theta(n^c)$ for some $c\in\mathbb{R}.$ Since the training procedure involves only elementary algebraic operations, the quantities there should be of powers of $n.$  In iteration $t$, the feature update without preconditioning is given by
\begin{equation*}
\begin{aligned}
\Delta f_t&:=f_t(x)-f_{t-1}(x)\\&=-\eta b_{t-1}^2(f_{t-1}(x)-y)\|x\|^2\\&\quad-\eta(a_{t-1}^Tx)^2(f_{t-1}(x)-y)\\&\quad +\eta^2(f_{t-1}(x)-y)^2b_{t-1}(a_{t-1}^Tx)\|x\|^2.
\end{aligned}
\end{equation*}
We denote $\delta_t^1=\eta b_{t-1}^2(f_{t-1}(x)-y)\|x\|^2,\delta_t^2=\eta(a_{t-1}^Tx)^2(f_{t-1}(x)-y),$ and $\delta_t^3=\eta^2(f_{t-1}(x)-y)^2b_{t-1}(a_{t-1}^Tx)\|x\|^2.$ For stable feature learning, we would like $\delta_t^1,\delta_t^2,\delta_t^3\in\Theta(1)$ and further $f_{t-1}(x)\in\Theta(1).$ Note that $\delta_t^3\in\Theta(1)$ is guaranteed as long as $\delta_t^1,\delta_t^2\in\Theta(1).$ Thus for stable feature learning, it suffices to have $\delta_t^1,\delta_t^2,f_{t-1}(x)\in\Theta(1).$ For the sake of notational clarity, we introduce new notation $\gamma$ such that 
 $v=\Theta(n^{\gamma[v]})$ captures the polynomial behavior for any $v$. We refer readers to Section A.2 in \cite{hayou2024lora} for additional properties of $\gamma[\cdot].$ Stable feature learning is thus equal to the following linear constraints 
\begin{equation*}
\begin{cases}
c+2\gamma[b_{t-1}]+1=0 \quad (\text{for } \delta_t^1=\Theta(1)),\\
c+2\gamma[a_{t-1}^Tx]=0\quad (\text{for } \delta_t^2=\Theta(1)),\\
\gamma[b_{t-1}]+\gamma[a_{t-1}^Tx]=0 \quad (\text{for } f_{t-1}(x)=\Theta(1)),
\end{cases}
\end{equation*}
from which we can derive $c=-1/2$ and thus the learning rate should be set to $\eta=\Theta(n^{-1/2}).$ With $\gamma[b_1]=\gamma[b_0]=0$ and $\gamma[a_1^Tx]=\gamma[\eta b_0y\|x\|^2]=1/2,$ one can inductively deduce that $\gamma[b_t]=0$ and $\gamma[a_t^Tx]=1/2$ for all $t$ and thus $\gamma[f_t]=1/2$, contradicting to $ f_t=\Theta(1).$ Instead, stable feature learning requires to set separately $\eta_a=\Theta(1/n)$ and $\eta_b=\Theta(1)$  as shown in Proposition 2 in \cite{hayou2024lora}. 

\textbf{Analysis of preconditioned training. } We now proceed to show that with our preconditoned parameter updates (\ref{formula}), $c=-1$ would generate stable feature learning. After the injection of preconditioner (\ref{formula}), the training dynamics are given by
\begin{equation*}
\begin{aligned}
\Delta f_t^{\text{}}&:=f_t^{\text{}}(x)-f_{t-1}^{\text{}}(x)\\
& =(b_{t-1}-\eta (f_{t-1}^{\text{}}(x)-y)a_{t-1}^Tx\underbrace{\|a_{t-1}^{\text{}}\|^{-2}}_{\text{ preconditioner}})(a_{t-1}^T\\&\qquad -\eta (f_{t-1}^{\text{}}(x)-y)b_{t-1}^{\text{}}x^T\underbrace{b_{t-1}^{-2}}_{\text{ preconditioner}})x\\
&=-\eta (f_{t-1}(x)-y)\|x\|^2\\&\qquad -\eta(a_{t-1}^Tx)^2(f_{t-1}(x)-y)\|a_{t-1}\|^{-2}\\&\qquad +\eta^2(f_{t-1}(x)-y)^2b_{t-1}^{-1}\|a_{t-1}\|^{-2}(a_{t-1}^Tx)\|x\|^2.
\end{aligned}
\end{equation*}
We similarly define $\delta_t^{1 \text{scaled}}=\eta(f_{t-1}(x)-y)\|x\|^2,$ $\delta_t^{2 \text{scaled}}=\eta(a_{t-1}^Tx)^2(f_{t-1}(x)-y)\|a_{t-1}\|^{-2},$ $\delta_t^{3 \text{scaled}}=\eta^2(f_{t-1}(x)-y)^2b_{t-1}^{-1}\|a_{t-1}\|^{-2}(a_{t-1}^Tx)\|x\|^2$ as scaled version of $\delta_t^1,\delta_t^2,\delta_t^3.$ For stable feature learning, we thus have the below modified linear constraints 
\begin{equation*}
\begin{cases}
c+1=0 \quad (\text{for } \delta_t^{1 \text{scaled}}=\Theta(1)),\\
c+2\gamma[a_{t-1}^Tx]-\gamma[\|a_{t-1}\|^2]=0\quad (\text{for } \delta_t^{1 \text{scaled}}=\Theta(1)),\\
\gamma[b_{t-1}]+\gamma[a_{t-1}^Tx]=0 \quad (\text{for } f_{t-1}^{\text{}}(x)=\Theta(1)),
\end{cases}
\end{equation*}
from which we can derive $c=-1.$ With $\eta^{\text{}}=\Theta(n^{-1}),$ we get $\gamma[b_1^{\text{}}]=\gamma[b_0]=0$ and $\gamma[a_1^{\text{}T}x]=\gamma[\eta^{\text{}}b_0^{-1}y\|x\|^2]=0.$ One can recursively derive $b_t^{\text{}},a_t^{\text{}T}x,\delta_1^{\text{scaled}},\delta_2^{\text{scaled}},\delta_3^{\text{scaled}}\in\Theta(1)$  for all $t$, which preserves $f_t\in\Theta(1)$ and $\Delta f_t\in\Theta(1)$, i.e., stability is achieved.

The above toy example illustrates that our preconditioned LoRA updates achieve stable feature learning with learning rates for $A$ and $B$ of same order of magnitude, while for unpreconditioned LoRA, different learning rates are required to obtain the same stability. Nevertheless, the toy example is limited to linear model with LoRA rank $r=1$ and also gradient updates without a momentum term. Indeed, the  stability persists for arbitrary LoRA ranks and for the Adam optimizer when aided with our preconditioner (\ref{formula}), which we formalize as our theorem below.  
\begin{theorem}\label{stable_thm}[Stable Feature Learning (Informal)] Assume LoRA parameters $A$ and $B$ are trained with Adam scaled by our preconditioner as in \eqref{formula}. Further assume that $BAx$ has dimension $\Theta(n).$ Then the LoRA model achieves  stable feature learning with $\eta=\Theta(1).$ While for unscaled Adam, $\eta_A=\Theta(n^{-1})$ and $\eta_B=\Theta(1)$ are required for stable feature learning.
\end{theorem}
\vspace{-0.2cm}
Note here we require $\eta=\Theta(1)$ instead of $\eta=\Theta(n^{-1})$ as in the toy example, this discrepancy is due to our assumption about vector output and  Adam optimizer's gradient processing.  See Section \ref{stable_proof} for the detailed statement and the proof of Theorem \ref{stable_thm} as well as an explanation of this difference. Theorem 1 in \cite{hayou2024lora} shows that $\eta_A=\Theta(n^{-1})$ and $\eta_B=\Theta(1)$ are required for LoRA training with unpreconditioned Adam to obtain stability, revealing that without preconditioning we need to tune $\eta_A$ and $\eta_B$ separately while our preconditioner elegantly fixes this imbalance between the learning rates.

\begin{algorithm*}[ht!]
\caption{Pseudocode of scaled AdamW in  PyTorch.}
\label{algo:code}

\algcomment{\fontsize{7.2pt}{0em}\selectfont \texttt{pairwise}: read every two elements in a list

}
\definecolor{codeblue}{rgb}{0.25,0.5,0.5}
\definecolor{codered}{rgb}{0.8,0.2,0.2}
\lstset{
  backgroundcolor=\color{white},
  basicstyle=\fontsize{7.2pt}{7.2pt}\ttfamily\selectfont,
  columns=fullflexible,
  breaklines=true,
  captionpos=b,
  commentstyle=\fontsize{7.2pt}{7.2pt}\color{codeblue},
  keywordstyle=\fontsize{7.2pt}{7.2pt},
  emph={{,},]}, 
emphstyle=\color{codered}, 
escapeinside={<@}{@>}
}
\begin{lstlisting}[language=python]
# group trainable parameters into LoRA pairs in train.py. 
<@\textcolor{codered}{for LoRA$\_$A, LoRA$\_$B in pairwise(trainable$\_$parameter):}@>
     <@\textcolor{codered}{param$\_$groups.append($\{$"params": [LoRA$\_$A,LoRA$\_$B], "lr": learning$\_$rate$\}$)}@>
# apply preconditioner in optimizer.py
for group in param_groups:
      A, B = group["params"]  
      dA, dB = group["params"].grad 
      # update parameter A 
      <@\textcolor{codered}{dA$\_$scaled =inverse(B.T@B+delta*torch.eye(r)).mm(dA) }@> # precondition gradient of A
      A_m = beta1*A_m+(1-beta1)*dA_scaled; A_m_hat = A_m/(1-beta1**t)  # update first momentum of A
      A_v = beta2*A_v+(1-beta2)*dA_scaled**2; A_v_hat = A_v/(1-beta2**self.t)  # update second momentum of A
      A.add_(A_m_hat/(sqrt(A_v_hat)+eps), -group['lr']) # update A
      # update parameter B similarly
      # ...
\end{lstlisting}
\end{algorithm*}
\section{\textcolor{black}{A Riemannian Metric Formulation }}\label{metric_sec}
After discussing the motivation and superiority for stabilizing feature learning of our proposed preconditioner in prior sections, we now review how the proposed preconditioner is derived from a new Riemannian metric. Optimization over matrix with rank constraint is a common example for optimization over Riemannian submanifolds. Specifically, matrices with fixed rank form a quotient manifold of general matrix field. Let $\mathcal{M}$ denote any Riemannian submanifold, then  Riemannian gradient descent usually takes form \[X_{t+1}=\mathcal{R}(X_t-\eta \nabla_{M_r}f(X_t)),\]
where $\mathcal{R}$ is a retraction operator that maps to $\mathcal{M}$. Here, $f$ is the objective function and  $\nabla_{M_r}f(X_t)$ denotes Riemannian gradient defined by \[Df(x)[\eta_x]=g_x(\nabla_{M_r}f(x),\eta_x) \text{ for all }\eta_x\in T_x\mathcal{M},\]
where $Df(x)[\eta_x]$ is the conventional Euclidean directional derivative
of $f$ in the direction $\eta_x$. In this definition, $g_x$ is a Riemannian metric which maps two elements in the tangent space $T_x\mathcal{M}$ to a real number and might not be unique, we will see that the innovative design of $g_x$ is the key to derive our preconditioned updates (\ref{formula}). Before proceeding to the derivation of our preconditioner, we need one more piece of knowledge about quotient manifold.

When it comes to a quotient space $\mathcal{M}/\sim$ where each element $[x]=\{y\in\mathcal{M}:y\sim x\}$ represents an equivalence class, for low-rank matrix problems where $AB^T$ are considered, each $(A,B)$ pair is equivalent to $(AO,BO^{-1})$ for any $O\in GL(r)$ in the sense that they obtain the same objective value, where $GL(r)$ stands for the general linear group over $r\times r$ invertible matrices.  Tangent space $T_{[x]}\mathcal{M}$ respects the equivalence relation $\sim$ by the introduction of horizontal and vertical spaces at each element, i.e., we decompose $T_x\mathcal{M}=\mathcal{V}_x\oplus\mathcal{H}_x$ where $\mathcal{V}_x$ is the tangent space of the equivalence class $[x]$ and $\mathcal{H}_x$ is its complement. Then each $\eta_{[x]}\in T_{[x]}\mathcal{M}$ corresponds to a unique element in $\mathcal{H}_x$ which is called the horizontal lift of $x.$ A Riemannian metric for $\mathcal{M}/\sim$ satisfies 
\[g_{[x]}(\eta_{[x]},\xi_{[x]})=g_x(\eta_x,\xi_x),\]
where $\eta_x$ and $\xi_x$ are horizontal lifts of $\eta_{[x]}$ and $\xi_{[x]}$ at $x$. Thus a Riemannian metric on quotient space is invariant along equivalence classes of the quotient space.

In \cite{Mishra_2016}, Mishra et al. describe a new Riemannian metric that draws motivation from regularized Lagrangian and involves both objectives and constraints. When specialized to least squares matrix decomposition  problem of form $\|AB^T-Y\|_F^2/2$, following derivation (33) in \cite{Mishra_2016}, we get the following metric on quotient space $[x]=(A,B)$,
\begin{equation}\label{new_metric}
g_{[x]}(\eta_{[x]},\xi_{[x]})=\langle\eta_A,\xi_AB^TB\rangle +\langle\eta_B,\xi_B A^TA\rangle.
\end{equation}
Under this new metric, the Riemannian gradient descent effectively replaces the  gradient operators via the map
\begin{equation*}\label{preconditioner}
\Big(\frac{\partial}{\partial A}\Big)\rightarrow \Big(\frac{\partial}{\partial A}\Big)(B^TB)^{-1}, \Big(\frac{\partial}{\partial B}\Big)\rightarrow \Big(\frac{\partial}{\partial B}\Big)(A^TA)^{-1},
\end{equation*}
which then corresponds to the preconditioned updates (\ref{formula}) we propose in this work. For details in the design of the new metric (\ref{new_metric}) and its connection with sequential quadratic programming (SQP), we point readers to \cite{Mishra_2016} and \cite{mishra2012riemannian} which include a thorough explanation and visualization of Riemannian optimization concepts.
\section{Empirical Results}\label{simulation}
\subsection{Algorithms and Simple Implementation}
In this section, we describe the optimization algorithms and the software  implementation we use for accelerating LoRA training in practice. Let $\mathcal{L}$ denote the loss function and $(A^{(t)},B^{(t)})$ denote the pair of LoRA parameters in the $t$-th iteration. 

\textbf{Scaled GD.} To apply gradient scaling to SGD, we follow exactly (\ref{formula}) and use $(B^{(t-1)^T}B^{(t-1)}+\delta I)^{-1}$ to scale gradient  $\nabla_{A^{(t-1)}} \mathcal{L}$ and vice versa. Note here a small $\delta>0$ is used to tackle the case when either $B^{(t-1)^T}B^{(t-1)}$ or $A^{(t-1)}A^{(t-1)^T}$ is not invertible. See Appendix \ref{sgd_sec} for the complete algorithm.

\textbf{Scaled AdamW.} The conventional scaled GD method studied for classic low-rank matrix optimization problems is only based on the gradient descent method. Our Theorem \ref{stable_thm} introduces this preconditioner for Adam for the first time and reveals its advantages for Adam and its variants. We note that AdamW is more popular than SGD for fine-tuning due to its fast convergence. To extend preconditioning to AdamW, one could apply the preconditioner at each individual gradient computation step or  apply the scaling to the processed gradient.  Though our proof of Theorem \ref{stable_thm} adopts the latter version and preconditions the processed gradient,  we empirically  find that scaling the gradient in each single iteration behaves better than scaling the processed gradient, thus we scale each single gradient in AdamW algorithm for our practical implementation and dub it \text{scaled AdamW} method. 
We outline the pseudocode of our scaled AdamW in Algorithm \ref{algo:code}. Remarkably, our method only requires four lines change of existing optimizer code, which is simple to implement. We highlight the changed code in red color.

\subsection{Runtime Comparison}\label{runtime_section}
\begin{table*}[ht!]
\centering
{
\scalebox{1.0}{
\begin{tabular}{c|cccccc}
\toprule
\multirow{2}{*}{Method}  
 & \multicolumn{5}{c}{E2E}  \\ 
& BLEU & NIST &MET & ROUGE-L & CIDEr \\
  \midrule 
SGD$_{r=4}$ & 66.6 & 8.54& 44.2 & 68.2 & 2.32 \\ 
scaled GD (ours)$_{r=4}$ & 69.2 & 8.71 & 46.3 & 70.9 & 2.48 \\ 
\midrule
  AdamW$_{r=4}$ & 68.9 & 8.69 & 46.5&  71.3 & 2.51  \\ 
  scaled AdamW (ours)$_{r=4}$ & \textbf{69.6} & \textbf{8.77} & \textbf{46.6} & \textbf{71.8} & \textbf{2.52} \\
   \bottomrule 
\end{tabular}}
}
\caption{Scores for LoRA fine-tuning of GPT-2 medium model on E2E Natural Language Generation challenge 
 with different optimizers. Our scaled optimizers outperform unscaled optimizers on all evaluation metrics and scaled GD closes the performance gap between SGD and AdamW.  See Section \ref{gpt2_section} for experimental details.}\label{table1}
\end{table*} 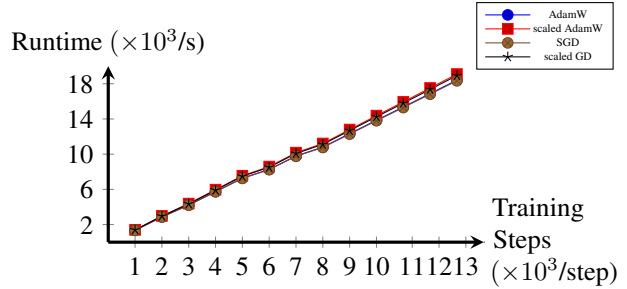
\begin{figure}
\begin{center}
\begin{tikzpicture}
\begin{axis}[
scale = 0.53,
xmin = 0, xmax = 21,
ymin = 0, ymax = 10,
axis lines* = left,
xtick = \empty, ytick = \empty,
clip = false,
axis line style ={very thick},
width=11cm,height=6cm,
axis lines = middle,
xtick={1.5,3,4.5,6,7.5,9,10.5,12,13.5,15,17,18.5,20},
xticklabels={1,2,3,4,5,6,7,8,9,10,11,12,13},
ytick = {1,3,5,7,9},
yticklabels = {2,6,10,14,18},
legend style={anchor=south west, nodes={scale=0.43, transform shape}}
]
 \addplot coordinates{
        (1.5,10*1345.1330564022064/20000)
        (3,10*2869.9201140403748/20000)
        (4.5,10*4220.749669790268/20000)
        (6,10*5746.169511318207/20000)
        (7.5,10*7273.201901912689/20000)
        (9,10*8270.15891122818/20000)
        (10.5, 10*9796.859780311584/20000)
        (12,10*10802.082231283188/20000)
        (13.5,10*12327.054862976074/20000)
        (15,10*13853.072108268738/20000)
        (16.5,10*15358.31376695633/20000)
        (18,10*16885.04076051712/20000)
        (19.5,10*18410.837102651596/20000)
   };
   \addlegendentry{AdamW};

   \addplot coordinates{
      (1.5,10*1402.005631685257/20000)
        (3,10*2982.695142030716/20000)
        (4.5,10*4382.921616315842/20000)
        (6,10*5966.059769392014/20000)
        (7.5,10*7549.9703640937805/20000)
        (9,10* 8589.653812408447/20000)
        (10.5, 10*10171.349878072739/20000)
        (12,10*11212.527932882309/20000)
        (13.5,10*12796.21893620491/20000)
        (15,10*14379.597594499588/20000)
        (16.5,10*15938.074256896973/20000)
        (18,10*17519.454837083817/20000)
        (19.5,10*19102.238866090775/20000)
   };
   \addlegendentry{scaled AdamW };

   \addplot coordinates{
      (1.5,10*1339.1964378356934/20000)
        (3,10*2857.334985256195/20000)
        (4.5,10*4204.037150144577/20000)
        (6,10*5723.812128782272/20000)
        (7.5,10*7243.765990495682/20000)
        (9,10*8236.28890967369/20000)
        (10.5, 10*9755.06168961525/20000)
        (12,10*10756.093092441559/20000)
        (13.5,10*12275.77213048935/20000)
        (15,10*13795.70438528061/20000)
        (16.5,10*15293.724370241165/20000)
        (18,10*16811.67221570015/20000)
        (19.5,10*18330.58613562584/20000)
   };
   \addlegendentry{SGD};

   \addplot coordinates{
      (1.5,10*1387.3638546466827/20000)
        (3,10*2953.948998451233/20000)
        (4.5,10*4344.553594350815/20000)
        (6,10*5912.983655452728/20000)
        (7.5,10*7480.2507157325745/20000)
        (9,10*8508.386858463287/20000)
        (10.5, 10*10076.736971616745/20000)
        (12,10*11109.620853185654/20000)
        (13.5,10*12678.96171617508/20000)
        (15,10*14246.523627758026/20000)
        (16.5,10*15790.976345539093/20000)
        (18,10*17359.99435400963/20000)
        (19.5,10*18926.015352249146/20000)
   };
   \addlegendentry{scaled GD };




\node [right] at (current axis.right of origin)  (A) {};
    \node        (B) [right of=A,text width=2cm]{Training Steps
$(\times 10^3$/step)};;
    \path[-] (A) edge (B);
\node [above] at (current axis.above origin) {Runtime $(\times 10^3$/s)};
\end{axis}
\end{tikzpicture}
\end{center}
\caption{Runtime for LoRA fine-tuning GPT-2 medium model with different optimizers. Our scaled methods introduce negligible runtime overhead and train as fast as unscaled methods. See Section \ref{runtime_section} for experimental details. Here we set $r=4.$}\label{small_GPU}
\end{figure}
A main concern of our method may arise from the common stereotype about the cumbersomeness and heavy computation complexity of usual preconditioning methods. This is likely due to Hessian-inverse type second-order preconditioner usually involves large size preconditioners and complex computation procedures, which is not the case for the preconditioners we consider. In each iteration,  we use current value of $(AA^T)^{-1}$ to precondition gradient of $B$. The preconditioner is easily obtained and is of small size. We present a runtime comparison between scaled optimizers and their unscaled counterparts for fine-tuning a GPT-2 medium model on E2E NLG challenge, see Section \ref{gpt2_section} for experimental details. Figure \ref{small_GPU} shows the runtime used for different optimizers for the fine-tuning task trained on NVIDIA A100 GPUs. Note there is little difference between the scaled optimizers and unscaled ones, which verifies that our preconditioner is practical.  See runtime comprisons for larger rank $r=256$ in Appendix \ref{runtime_sec}.
\subsection{LLM Fine-Tuning}
In this section, we study the fine-tuning task for GPT-2 model and Mistral 7B model with our  scaled optimizers. Empirically, we observe that our scaled optimizers outperform unscaled optimizers by a large margin for various tasks, datasets, LoRA ranks,  model sizes, model types, and benchmarks, which demonstrates the superiority of gradient scaling in training LoRA models. See below and also Appendix \ref{simulation_supp} for all our experiments.
\subsubsection{GPT-2}\label{gpt2_section}
We exploit the new preconditioner for LoRA fine-tuning of GPT-2 models. We follow exactly the same experimental  setup as  \cite{hu2021lora} except that here we tune learning rate individually using grid search for different methods being tested, see Appendix \ref{gpt2_supp} for experimental details and training hyperparameters. Table \ref{table1} shows the final score for fine-tuning GPT-2 medium model with LoRA rank $4$ on E2E \cite{novikova2017e2e} natural language generation challenge. It can be observed that  scaled GD method closes the gap between SGD and AdamW and behaves comparable to AdamW while demanding smaller optimizer  storage. Scaled AdamW method improves score of  AdamW method on all evaluation metrics, which reveals that the new preconditioner is advantageous even for gradient computation normalized by gradient variance in AdamW method.  See also  Appendix \ref{gpt2_supp} for experimental results for different LoRA ranks, different datasets, and different model sizes. Our scaled optimizers show significant and uniform improvements over unscaled optimizers for almost all tests.
\subsubsection{\textcolor{black}{Mistral 7B}}\label{mistral_mainsec}
Mistral 7B is a recent language model released by the Mistral AI team \cite{jiang2023mistral} which has been shown to outperform Llama 2-13B on all benchmarks and Llama 1-34B on many benchmarks \cite{Mistral_Announcement}, and thus is considered the most powerful language models for its size to date. We experiment our scaled optimizers with this new language model on the GLUE  \cite{wang-etal-2018-glue} benchmark for natural language understanding problems. 
Table \ref{mistral_result} shows
 the final fine-tuning results. Notably that our scaled optimizers outperform unscaled optimizers on all evaluation metrics, which demonstrates the effectiveness of gradient scaling for fine-tuning Mistral 7B model. See Appendix \ref{mistral_append_sec} for training hyperparameter selection. 
\begin{table*}
\centering
{
\scalebox{1.0}{\begin{tabular}{l|cccccccccc}
 \toprule
\multirow{2}{*}{Method} & 
  \multicolumn{10}{c}{GLUE}  \\ 
 &    \textcolor{black}{MNLI} & \textcolor{black}{SST-2} & \textcolor{black}{MRPC} & \textcolor{black}{CoLA} & \textcolor{black}{QNLI} & \textcolor{black}{QQP} & \textcolor{black}{RTE} & \textcolor{black}{STS-B} & \textcolor{black}{WNLI} & Avg.\\
 \midrule
  
  SGD$_{r=16}$  & 88.15 & 96.10 & 70.10 & 55.89 & 94.22 & 88.59 & 50.90 & 47.64 & 49.30 & 71.21  \\ 
scaled GD (ours)$_{r=16}$ & 90.21 & 96.90 & 81.62 & 68.17 & 94.40 & 91.15 & 54.15 & 90.31 & 56.34 &80.36 \\ 
  AdamW$_{r=16}$  & 89.86 & 96.79 & 88.48 & 71.05 & 94.42& 91.24 & 90.61 & 90.42 & 81.69 & 88.28\\ 
  scaled AdamW (ours)$_{r=16}$  & \textbf{90.68} & \textbf{97.25} & \textbf{89.46} & \textbf{71.30} & \textbf{94.67}& \textbf{92.22} & \textbf{91.34}& \textbf{91.10}& \textbf{83.10} & \textbf{89.01}\\ 
  \bottomrule
\end{tabular}}
}
\caption{Scores for LoRA fine-tuning of 4-bit quantized Mistral 7B model on GLUE benchmark for Natural Language Understanding (NLU) challenges with different optimizers. Our scaled optimizers outperform unscaled optimizers on all evaluation metrics. See Section \ref{mistral_mainsec} for experimental details.}\label{mistral_result}
\end{table*}
\subsection{Diffusion Model Fine-Tuning}
Diffusion models are now used for various image generation tasks and   LoRA has been widely used for fine-tuning diffusion models. Here we start with the commonly used    Stable Diffusion V1.5 model and show the effectiveness of applying our new preconditioner in LoRA fine-tuning for object generation. Then, we  experiment with the Mix-of-Show model \cite{gu2023mixofshow} which can generate high-quality face images. We observe that with gradient scaling, image generation becomes much more robust to learning rate changes, which is a reflection of the fact that our new preconditioner stabilizes the training process against learning rate variations. This has important practical benefits since learning rate choices can be crucial in image generation problems where small difference in learning rate can produce images of very different quality. This can be observed from Figures \ref{vase_main} and  \ref{diffusion_main}. Furthermore, it's widely  observed that training loss is useless in monitoring image generation quality when training diffusion models. Thus a better optimizer which is more robust to learning rate choices is very important. See Appendix \ref{diffusion_supp} for experimental details for diffusion model fine-tuning tasks.
\subsubsection{\textcolor{black}{Object Generation}}\label{object_section}
We build our object generation experiments on a popular stable diffusion fine-tuning repository \cite{Ryu2023} with Stable Diffusion V1.5 as the base model. We follow the default settings there and tune both the U-Net and the text encoder where LoRA parameters are injected. For all experiments, we fix the U-Net fine-tuning learning rate as default value $1e-4$ which we find important for generating recognizable images. After fine-tuning on $6$ images of a red vase  titled ``a photo of $\langle V_{\text{vase}} \rangle$'',    Figure \ref{vase_main} shows the generation results for prompt ``a blue $\langle V_{\text{vase}} \rangle$". With large learning rate as $1e-2$ for text encoder fine-tuning, AdamW produces out-of-distribution results while our method produces satisfactory images. With default learning rate setting, AdamW still fails to capture the prompt information and generates only red vases. Instead,  scaled AdamW with default learning rate is able to produce the desired blue vase.  AdamW turns out to be able to generate the desired blue vase for learning rate value such as $1e-6$. See Appendix \ref{stable_diffusion_append} for other target object generation including chairs and dogs. Scaled AdamW improves AdamW for all experiments.
\subsubsection{Face 
Generation}\label{face_gen}
\begin{figure*}[ht!]
 \hspace{0.8cm}
\includegraphics[width=0.75\linewidth]{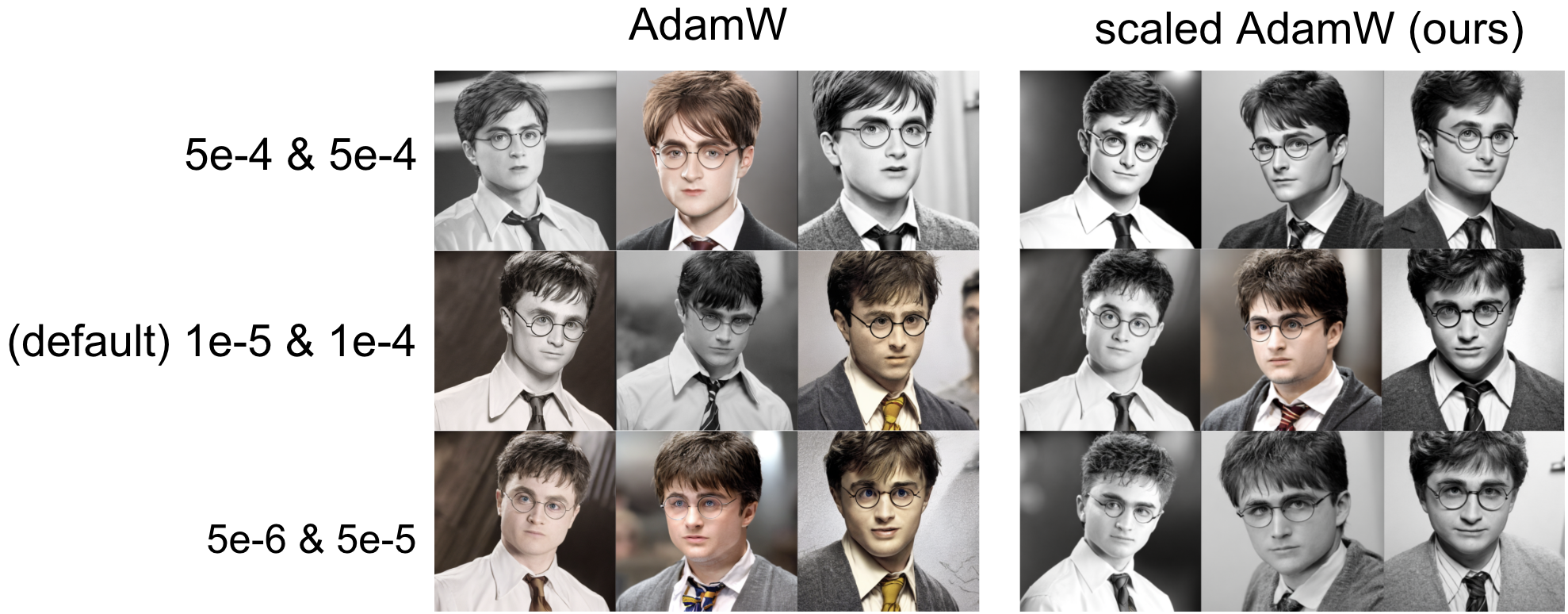}
\caption{Generation results for prompt ``a pencil sketch of $\langle V_{\text{potter}} \rangle$" by Mix-of-Show model with different optimizers and various learning rates. Our method
(scaled AdamW) generates photos better capturing the prompt, i.e., a pencil sketch, and is more robust to learning rate choices. See Section \ref{face_gen} for experimental details.}\label{diffusion_main}
\end{figure*}
\vspace{-0.2cm}
Face generation is a more challenging task compared to object generation and we thus switch to Mix-of-Show \cite{gu2023mixofshow} variant of custom diffusion model which is originally designed for multi-concept LoRA  and has been recognized to be able to generate high-quality face images.  For better visualization of differences between different LoRA optimization methods, we turn off embedding fine-tuning and tune only text-encoders and U-Nets where LoRA factors are injected. 
We use $14$ images of Potter provided in the original project repository where the character name is replaced with $\langle V_{\text{potter}} \rangle$ in captions of the training images. Figure \ref{diffusion_main} shows generation results for prompt ``a pencil sketch of $\langle V_{\text{potter}} \rangle$" for various step sizes. Our method (scaled AdamW) generates visually better images more resemble a pencil sketch, which demonstrates its effectiveness in generating images of higher quality and also its robustness to learning rate changes. See Appendix \ref{gen_quality} for generation results for more prompts and also for the Hermoine character with different LoRA parameter fusion coefficients. See also Appendix \ref{lr_compare} for generation results including SGD and scaled GD methods for varying learning rates. Our observations are similar in all these generation results.

\section{Convergence Theory}\label{theory}
In this section, we further verify the superiority of scaled GD method applied to a reparameterized two-layer ReLU NN tuning problem, i.e., we show scaled GD method has convergence rate independent of data condition number of this specific problem and is thus advantageous compared to plain gradient descent. When scaled GD is applied to deep learning models,  nonlinearities in such models typically render theoretical analysis intractable. To tackle this, we instead study an equivalent problem to the two-layer ReLU neural network. We first introduce the concept of hyperplane arrangement matrices. For a data matrix $X\in\mathbb{R}^{n\times d}$  and any arbitrary vector $u\in\mathbb{R}^d$, We consider the set of diagonal matrices 
 \[\mathcal{D}:=\{\text{diag}(\indicfcn\{Xu\geq 0\})\},\] 
 which takes value $1$ or $0$ along the diagonals that indicate the set of possible arrangement activation patterns for the ReLU activation. Indeed, we can enumerate the set of sign patterns as $\mathcal{D}=\{D_{i}\}_{i=1}^{P}$ where $P$ is bounded by
 \[P\leq 2r\left(\frac{e(n-1)}{r}\right)^r\]
for $r=\text{rank}(X)$ \citep{pilanci2020neural,stanley2004introduction}. The two-layer ReLU model is equivalent to the problem below for squared loss through convexification under mild conditions\footnote{The equivalence holds when the number of hidden neurons is greater than or equal to $2Pd.$} \cite{mishkin2022fast},
\[\min_{W_i}\frac{1}{2}\|\sum_{i=1}^P D_iXW_i-Y\|_F^2.\]
Therefore, we base our analysis on fine-tuning the above model and show that the convergence rate of problem below with scaled GD method (Algorithm \ref{alg:scaled_gd}) has no dependence on condition number of the data matrix $X$. We focus on
\begin{equation}\label{main_prob}
\min_{A_i,B_i} \frac{1}{2}\|\sum_{i=1}^P D_iX(W_i+A_iB_i^T)-Y\|_F^2,
\end{equation}
where $X\in\mathbb{R}^{n\times d}, A_i\in\mathbb{R}^{d\times r},B_i\in\mathbb{R}^{c\times r},Y\in\mathbb{R}^{n\times c}.$ We consider the response model {$Y=\sum_{i=1}^P D_iX(W_i+A_\star^i{B_\star^i}^T)$}. Here, $X_\star^i=A_\star^i{B_\star^i}^T=U_\star^i\Sigma_\star^i{V_\star^i}^T$ are fixed and unknown matrices with $U_\star^i\Sigma_\star^i{V_\star^i}^T$ being the singular value decomposition of $X_\star^i$. Denote $F_\star^i=[A_\star^i, B_\star^i]^T$ and $F_t^i=[A_t^i,B_t^i]^T$ with $A_t^i$ and $B_t^i$ denote the value of $(A_i,B_i)$ at $t$-th iteration. Let $\sigma_r(\cdot)$ denote the $r$-th largest singular value. 

We first introduce the definition of Restricted Isometry Property (RIP) and illustrate the assumptions required for our theorem to hold,
\begin{definition} (RIP \cite{Recht_2010})
The matrix $A\in\mathbb{R}^{n\times d}$ satisfies rank-$r$ RIP with a constant $\delta_r\in[0,1)$ if for all matrices $M\in\mathbb{R}^{d\times c}$ of rank at most $r$, the below holds,
\[(1-\delta_r)\|M\|_F^2\leq \|AM\|_F^2\leq (1+\delta_r)\|M\|_F^2.\]
\vspace{-0.3cm} 
\end{definition}

\textcolor{black}{\vspace{-1cm} \begin{assumption}\label{main_assum}
Suppose that $D_iX$ obeys the $2r$-RIP with a constant $\delta_{2r}^i$ for each $i$, and $\|X^T{D_i}^TD_jX\|_2\leq \min(\frac{\delta_{2r}^i\|X_\star^i\|_F}{P\|X_\star^j\|_F},\frac{0.12}{7P(P+1)})$ for any $j\neq i.$
\end{assumption}
Note for matrix $X$ with i.i.d Gaussian entries $\mathcal{N}(0,1/d\|D_i\|_0)$, $D_iX$ satisfies RIP for a constant $\delta$ when $\|D_i\|_0$ is on the order of $r(d+c)/(d\delta^2)$. See \cite{Recht_2010} for other measurement ensembles satisfying the RIP property.
}  Note also $ \|X^T{D_i}^TD_jX\|_2\leq \|X^TX\|_2$ for all $(i,j)$'s. Thus bounding $\|X^T{D_i}^TD_jX\|_2$ amounts to bounding largest singular value of empirical covariance matrix. 
We consider a specific initialization strategy here which is an extension of spectral initialization for multiple terms as below,
\begin{definition}\label{spectral_init} (Extended Spectral Initialization) Let $A_0^i{B_0^i}^T$ be the best rank-$r$ approximation of $(D_iX)^T(Y-\sum_{j=1}^P D_jXW_j)$ for each $i.$
\end{definition}

Now,  we are ready to state our main convergence result as follows, 
\begin{theorem}\label{main_thm}
Under Assumption \ref{main_assum} with $\delta_{2r}^i\leq 0.01$ for each $i$. With extended spectral initialization described  in Definition \ref{spectral_init},   $\|A_t^i{B_t^i}^T-X_\star^i\|_F\leq 1.5\text{dist}(F_t^i,F_\star^i)$. In addition, if the step size $0<\eta\leq 2/3$, then the $(t+1)$-th iteration $F_{t+1}^i$ satisfies
\[\max_i(\text{dist}(F_{t+1}^i, F_\star^i))\leq (1-0.5\eta)\max_i(\text{dist}(F_{t}^i, F_\star^i)).\]
\end{theorem}
\begin{proof}
See Appendix \ref{proof}.
\end{proof}
Our result mainly builds on results from \cite{tong2021accelerating} and can be viewed as an extension of matrix sensing to ReLU neural networks.

\vspace{-0.2cm}
\section{Literature Review}\label{prior_work}
Our work is closely related to low-rank matrix optimization and we briefly review some basic knowledge and related work in Section \ref{riemannian_overview}. Our work applies preconditioners for accelerating the LoRA fine-tuning process, which falls into  preconditioning methods for PEFT, and related prior work there is discussed in Section \ref{precond_deep} and Section \ref{peft}.
\vspace{-0.1cm}
\subsection{Riemannian Optimization}\label{riemannian_overview}
\vspace{-0.1cm}
Several recent studies have made theoretic contributions to the convergence rate of scaled GD method which employs the preconditioner (\ref{formula}). Specifically, in \cite{Tong_2021, tong2021accelerating}, the authors show local convergence of scaled GD method with better convergence rate independent of data condition number compared to plain gradient descent method for some classic low-rank matrix optimization problem including matrix sensing, robust PCA, etc. The authors of \cite{jia2023preconditioning} show global convergence of scaled GD method with rate independent of data condition number for least squares matrix decomposition problem $\|AB^T-Y\|_F^2/2$. Different variants of scaled GD have been proposed and studied. In \cite{zhang2023preconditioned}, the authors suggest to use $(A^TA+\lambda I)^{-1}$ and $(B^TB+\lambda I)^{-1}$ with some fixed $\lambda>0$ in replace of (\ref{formula}) for tackling overparametrization and ill-conditioness in matrix sensing problems. In \cite{zhang2023fast}, the authors suggest using $(A^TA+\lambda_t I)^{-1}$ where $\lambda_{t+1}=\beta \lambda_t$ (similar change for $B$), i.e., using an exponentially decay regularization term.  \cite{zhang2023preconditioned} proposes precGD method which sets $\lambda_t=\sqrt{f(A_tB_t^T)}$. \cite{tong2022scaling, ma2023provably} present extension of scaled GD method to tensor optimization problem. \cite{jia2023preconditioning} analyzes AltScaledGD method which updates $A$ and $B$ alternatively and shows that such method has better convergence rate for larger step size.
\vspace{-0.1cm}
\subsection{Preconditioners in Deep Learning}\label{precond_deep}
\vspace{-0.1cm}
Current deep learning training is dominated by gradient-based method which follows a  descent direction to update parameters for decreasing objective value. For accelerating such training procedure, more advanced techniques such as Adagrad \cite{JMLR:v12:duchi11a} proposes to scale gradient based on their variance. Specifically, $G_t^{-1/2}$ is used as gradient preconditioner where $G_t$ is accumulated outer product of historic 
 subgradients. More practical optimizers such as Adam \cite{kingma2017adam} and AdamW \cite{loshchilov2019decoupled} perform like a diagonal version of Adagrad and are the main training tools for most deep learning models in various fields. More recently, Shampoo \cite{gupta2018shampoo} has been proposed which uses a left preconditioner and a right preconditioner for a weight matrix.  Shampoo is in spirit close to Adagrad but requires much less storage. In contrast with the preconditioners designed for accelerating optimization procedure for general deep learning models. We study a specific preconditioner designed for LoRA fine-tuning model which exploits its low rank matrix factorization property and borrows from Riemannian optimization knowledge.
 \subsection{PEFT Fine-Tuning Review}\label{peft}
\vspace{-0.1cm}
Current commonly-used deep learning models are growing larger and larger, making full fine-tuning for downstream tasks nearly impossible. A line of parameter-efficient fine-tuning methods emerges and has been used in various fields. These methods aim at achieving low fine-tuning loss with fewer trainable parameters. One popular PEFT method is LoRA \cite{hu2021lora}, which proposes to add a low-rank adaptation to each existing weight matrix. By factorizing the update into two low-rank matrices, LoRA is able to achieve similar fine-tuning result as full fine-tuning with 10,000 times fewer parameters. LoRA has shown good performance in both language model fine-tuning and vision model fine-tuning. Variants of LoRA method involve DyLoRA \cite{valipour2023dylora}, IncreLoRA \cite{Zhang2023IncreLoRAIP}, and AdaLoRA \cite{zhang2023adaptive}, all focus on dynamically adjusting the rank 
 hyperparameter. GLoRA \cite{chavan2023oneforall} generalizes LoRA by introducing a prompt module; Delta-LoRA \cite{zi2023deltalora} proposes to simultaneously update pretrained model weights by difference of LoRA weights. QLoRA \cite{dettmers2023qlora} exploits quantized LoRA model which further reduces the model size. Besides such additive methods, there are also multiplicative PEFT methods such as the orthogonal fine-tuning method (OFT) \cite{qiu2023controlling} and its variant BOFT \cite{liu2023parameterefficient}.

Though LoRA has become very popular and different variants emerge,  current LoRA training mainly exploits gradient based optimizers and we are unaware of any prior work studying  acceleration of LoRA training given its special low-rank matrix factorization nature. Our work shows that by regrouping trainable parameters and applying an $r\times r$ preconditioner, the optimization procedure of LoRA can be significantly enhanced with negligible storage and runtime overhead.

\vspace{-0.1cm}
\section{\textcolor{black}{Conclusion}}\label{conclusion}
\vspace{-0.1cm}
In this work, we borrow tools from Riemannian optimization to enhance LoRA fine-tuning. Specifically, we study the application of a Riemannian gradient preconditioning method which introduces a new $r\times r$ preconditioner to LoRA fine-tuning procedure. Empirically, we observe that the gradient scaling boosts performance of both SGD and AdamW methods and theoretically we show that LoRA trained with preconditioned Adam method achieves stable feature learning under infinite-width NN setting while unpreconditioned training would require tuning learning rates for LoRA parameters separately.  Prior to our work, theoretic convergence for the proposed gradient scaling scheme has only been established for classic low-rank matrix optimization problems and only with gradient descent method, we first time introduces it to deep learning regime considering the low-rank nature of LoRA model and reveals its superiority beyond SGD.

\section*{Acknowledgements}
\vspace{-0.1cm}
This work was supported in part by the National Science Foundation (NSF) under Grant DMS-2134248; in part by the NSF CAREER Award under Grant CCF-2236829; in part by the U.S. Army Research Office Early Career Award under Grant W911NF-21-1-0242; and in part by the Office of Naval Research under Grant N00014-24-1-2164.
\vspace{-0.3cm}
\section*{Impact Statement}
\vspace{-0.1cm}
This paper aims to advance the field of machine learning through innovative research. While our work holds significant potential for societal impact, we do not identify any specific consequences that needs to be highlighted here.



\vspace{-0.2cm}
\bibliography{example_paper}

\begin{thebibliography}{44}
\providecommand{\natexlab}[1]{#1}
\providecommand{\url}[1]{\texttt{#1}}
\expandafter\ifx\csname urlstyle\endcsname\relax
  \providecommand{\doi}[1]{doi: #1}\else
  \providecommand{\doi}{doi: \begingroup \urlstyle{rm}\Url}\fi

\bibitem[Candes et~al.(2009)Candes, Li, Ma, and Wright]{candes2009robust}
Candes, E.~J., Li, X., Ma, Y., and Wright, J.
\newblock Robust principal component analysis?, 2009.

\bibitem[Chavan et~al.(2023)Chavan, Liu, Gupta, Xing, and Shen]{chavan2023oneforall}
Chavan, A., Liu, Z., Gupta, D., Xing, E., and Shen, Z.
\newblock One-for-all: Generalized lora for parameter-efficient fine-tuning, 2023.

\bibitem[Dettmers et~al.(2023)Dettmers, Pagnoni, Holtzman, and Zettlemoyer]{dettmers2023qlora}
Dettmers, T., Pagnoni, A., Holtzman, A., and Zettlemoyer, L.
\newblock Qlora: Efficient finetuning of quantized llms, 2023.

\bibitem[Duchi et~al.(2011)Duchi, Hazan, and Singer]{JMLR:v12:duchi11a}
Duchi, J., Hazan, E., and Singer, Y.
\newblock Adaptive subgradient methods for online learning and stochastic optimization.
\newblock \emph{Journal of Machine Learning Research}, 12\penalty0 (61):\penalty0 2121--2159, 2011.

\bibitem[Gal et~al.(2022)Gal, Alaluf, Atzmon, Patashnik, Bermano, Chechik, and Cohen-Or]{gal2022image}
Gal, R., Alaluf, Y., Atzmon, Y., Patashnik, O., Bermano, A.~H., Chechik, G., and Cohen-Or, D.
\newblock An image is worth one word: Personalizing text-to-image generation using textual inversion, 2022.

\bibitem[Gardent et~al.(2017)Gardent, Shimorina, Narayan, and Perez-Beltrachini]{gardent-etal-2017-webnlg}
Gardent, C., Shimorina, A., Narayan, S., and Perez-Beltrachini, L.
\newblock The {W}eb{NLG} challenge: Generating text from {RDF} data.
\newblock In Alonso, J.~M., Bugar{\'\i}n, A., and Reiter, E. (eds.), \emph{Proceedings of the 10th International Conference on Natural Language Generation}, pp.\  124--133, Santiago de Compostela, Spain, September 2017. Association for Computational Linguistics.
\newblock \doi{10.18653/v1/W17-3518}.
\newblock URL \url{https://aclanthology.org/W17-3518}.

\bibitem[Gu et~al.(2023)Gu, Wang, Wu, Shi, Chen, Fan, Xiao, Zhao, Chang, Wu, Ge, Shan, and Shou]{gu2023mixofshow}
Gu, Y., Wang, X., Wu, J.~Z., Shi, Y., Chen, Y., Fan, Z., Xiao, W., Zhao, R., Chang, S., Wu, W., Ge, Y., Shan, Y., and Shou, M.~Z.
\newblock Mix-of-show: Decentralized low-rank adaptation for multi-concept customization of diffusion models, 2023.

\bibitem[Gupta et~al.(2018)Gupta, Koren, and Singer]{gupta2018shampoo}
Gupta, V., Koren, T., and Singer, Y.
\newblock Shampoo: Preconditioned stochastic tensor optimization, 2018.

\bibitem[Hayou et~al.(2019)Hayou, Doucet, and Rousseau]{hayou2019impact}
Hayou, S., Doucet, A., and Rousseau, J.
\newblock On the impact of the activation function on deep neural networks training, 2019.

\bibitem[Hayou et~al.(2024)Hayou, Ghosh, and Yu]{hayou2024lora}
Hayou, S., Ghosh, N., and Yu, B.
\newblock Lora+: Efficient low rank adaptation of large models, 2024.

\bibitem[Hu et~al.(2021)Hu, Shen, Wallis, Allen-Zhu, Li, Wang, Wang, and Chen]{hu2021lora}
Hu, E.~J., Shen, Y., Wallis, P., Allen-Zhu, Z., Li, Y., Wang, S., Wang, L., and Chen, W.
\newblock Lora: Low-rank adaptation of large language models, 2021.

\bibitem[Jia et~al.(2023)Jia, Wang, Peng, Feng, and Meng]{jia2023preconditioning}
Jia, X., Wang, H., Peng, J., Feng, X., and Meng, D.
\newblock Preconditioning matters: Fast global convergence of non-convex matrix factorization via scaled gradient descent.
\newblock In \emph{Thirty-seventh Conference on Neural Information Processing Systems}, 2023.

\bibitem[Jiang et~al.(2023)Jiang, Sablayrolles, Mensch, Bamford, Chaplot, de~las Casas, Bressand, Lengyel, Lample, Saulnier, Lavaud, Lachaux, Stock, Scao, Lavril, Wang, Lacroix, and Sayed]{jiang2023mistral}
Jiang, A.~Q., Sablayrolles, A., Mensch, A., Bamford, C., Chaplot, D.~S., de~las Casas, D., Bressand, F., Lengyel, G., Lample, G., Saulnier, L., Lavaud, L.~R., Lachaux, M.-A., Stock, P., Scao, T.~L., Lavril, T., Wang, T., Lacroix, T., and Sayed, W.~E.
\newblock Mistral 7b, 2023.

\bibitem[Kingma \& Ba(2017)Kingma and Ba]{kingma2017adam}
Kingma, D.~P. and Ba, J.
\newblock Adam: A method for stochastic optimization, 2017.

\bibitem[Labonne(2024)]{Mistral_Fine_Tuning}
Labonne, M.
\newblock Mistral fine-tuning example, 2024.
\newblock URL \url{https://towardsdatascience.com/fine-tune-a-mistral-7b-model\\-with-direct-preference-optimization\\-708042745aac}.
\newblock Accessed: 2024-01-25.

\bibitem[Liu et~al.(2023)Liu, Qiu, Feng, Xiu, Xue, Yu, Feng, Liu, Heo, Peng, Wen, Black, Weller, and Schölkopf]{liu2023parameterefficient}
Liu, W., Qiu, Z., Feng, Y., Xiu, Y., Xue, Y., Yu, L., Feng, H., Liu, Z., Heo, J., Peng, S., Wen, Y., Black, M.~J., Weller, A., and Schölkopf, B.
\newblock Parameter-efficient orthogonal finetuning via butterfly factorization, 2023.

\bibitem[Loshchilov \& Hutter(2019)Loshchilov and Hutter]{loshchilov2019decoupled}
Loshchilov, I. and Hutter, F.
\newblock Decoupled weight decay regularization, 2019.

\bibitem[Lu et~al.(2022)Lu, Zhou, Bao, Chen, Li, and Zhu]{lu2022dpmsolver}
Lu, C., Zhou, Y., Bao, F., Chen, J., Li, C., and Zhu, J.
\newblock Dpm-solver: A fast ode solver for diffusion probabilistic model sampling in around 10 steps, 2022.

\bibitem[Ma et~al.(2023)Ma, Xu, Tong, and Chi]{ma2023provably}
Ma, C., Xu, X., Tong, T., and Chi, Y.
\newblock Provably accelerating ill-conditioned low-rank estimation via scaled gradient descent, even with overparameterization, 2023.

\bibitem[Mishkin et~al.(2022)Mishkin, Sahiner, and Pilanci]{mishkin2022fast}
Mishkin, A., Sahiner, A., and Pilanci, M.
\newblock Fast convex optimization for two-layer relu networks: Equivalent model classes and cone decompositions, 2022.

\bibitem[Mishra \& Sepulchre(2016)Mishra and Sepulchre]{Mishra_2016}
Mishra, B. and Sepulchre, R.
\newblock Riemannian preconditioning.
\newblock \emph{SIAM Journal on Optimization}, 26\penalty0 (1):\penalty0 635–660, January 2016.
\newblock ISSN 1095-7189.
\newblock \doi{10.1137/140970860}.
\newblock URL \url{http://dx.doi.org/10.1137/140970860}.

\bibitem[Mishra et~al.(2012)Mishra, Apuroop, and Sepulchre]{mishra2012riemannian}
Mishra, B., Apuroop, K.~A., and Sepulchre, R.
\newblock A riemannian geometry for low-rank matrix completion, 2012.

\bibitem[{Mistral AI team}(2023)]{Mistral_Announcement}
{Mistral AI team}.
\newblock Mistral-7b announcement, 2023.
\newblock URL \url{https://mistral.ai/news/announcing-mistral-7b}.
\newblock Accessed: 2024-01-25.

\bibitem[Nan et~al.(2021)Nan, Radev, Zhang, Rau, Sivaprasad, Hsieh, Tang, Vyas, Verma, Krishna, Liu, Irwanto, Pan, Rahman, Zaidi, Mutuma, Tarabar, Gupta, Yu, Tan, Lin, Xiong, Socher, and Rajani]{nan-etal-2021-dart}
Nan, L., Radev, D., Zhang, R., Rau, A., Sivaprasad, A., Hsieh, C., Tang, X., Vyas, A., Verma, N., Krishna, P., Liu, Y., Irwanto, N., Pan, J., Rahman, F., Zaidi, A., Mutuma, M., Tarabar, Y., Gupta, A., Yu, T., Tan, Y.~C., Lin, X.~V., Xiong, C., Socher, R., and Rajani, N.~F.
\newblock {DART}: Open-domain structured data record to text generation.
\newblock In Toutanova, K., Rumshisky, A., Zettlemoyer, L., Hakkani-Tur, D., Beltagy, I., Bethard, S., Cotterell, R., Chakraborty, T., and Zhou, Y. (eds.), \emph{Proceedings of the 2021 Conference of the North American Chapter of the Association for Computational Linguistics: Human Language Technologies}, pp.\  432--447, Online, June 2021. Association for Computational Linguistics.
\newblock \doi{10.18653/v1/2021.naacl-main.37}.
\newblock URL \url{https://aclanthology.org/2021.naacl-main.37}.

\bibitem[Novikova et~al.(2017)Novikova, Dušek, and Rieser]{novikova2017e2e}
Novikova, J., Dušek, O., and Rieser, V.
\newblock The e2e dataset: New challenges for end-to-end generation, 2017.

\bibitem[Pilanci \& Ergen(2020)Pilanci and Ergen]{pilanci2020neural}
Pilanci, M. and Ergen, T.
\newblock Neural networks are convex regularizers: Exact polynomial-time convex optimization formulations for two-layer networks, 2020.

\bibitem[Qiu et~al.(2023)Qiu, Liu, Feng, Xue, Feng, Liu, Zhang, Weller, and Schölkopf]{qiu2023controlling}
Qiu, Z., Liu, W., Feng, H., Xue, Y., Feng, Y., Liu, Z., Zhang, D., Weller, A., and Schölkopf, B.
\newblock Controlling text-to-image diffusion by orthogonal finetuning, 2023.

\bibitem[Radford et~al.(2019)Radford, Wu, Child, Luan, Amodei, and Sutskever]{Radford2019LanguageMA}
Radford, A., Wu, J., Child, R., Luan, D., Amodei, D., and Sutskever, I.
\newblock Language models are unsupervised multitask learners.
\newblock 2019.
\newblock URL \url{https://api.semanticscholar.org/CorpusID:160025533}.

\bibitem[Recht et~al.(2010)Recht, Fazel, and Parrilo]{Recht_2010}
Recht, B., Fazel, M., and Parrilo, P.~A.
\newblock Guaranteed minimum-rank solutions of linear matrix equations via nuclear norm minimization.
\newblock \emph{SIAM Review}, 52\penalty0 (3):\penalty0 471–501, January 2010.
\newblock ISSN 1095-7200.
\newblock \doi{10.1137/070697835}.
\newblock URL \url{http://dx.doi.org/10.1137/070697835}.

\bibitem[Rombach et~al.(2022)Rombach, Blattmann, Lorenz, Esser, and Ommer]{Rombach_2022_CVPR}
Rombach, R., Blattmann, A., Lorenz, D., Esser, P., and Ommer, B.
\newblock High-resolution image synthesis with latent diffusion models.
\newblock In \emph{Proceedings of the IEEE/CVF Conference on Computer Vision and Pattern Recognition (CVPR)}, pp.\  10684--10695, June 2022.

\bibitem[Ryu(2023)]{Ryu2023}
Ryu, S.
\newblock Low-rank adaptation for fast text-to-image diffusion fine-tuning.
\newblock \url{https://github.com/cloneofsimo/lora/tree/master}, 2023.

\bibitem[Schoenholz et~al.(2017)Schoenholz, Gilmer, Ganguli, and Sohl-Dickstein]{schoenholz2017deep}
Schoenholz, S.~S., Gilmer, J., Ganguli, S., and Sohl-Dickstein, J.
\newblock Deep information propagation, 2017.

\bibitem[Stanley et~al.(2004)]{stanley2004introduction}
Stanley, R.~P. et~al.
\newblock An introduction to hyperplane arrangements.
\newblock \emph{Geometric combinatorics}, 13\penalty0 (389-496):\penalty0 24, 2004.

\bibitem[Tong et~al.(2021{\natexlab{a}})Tong, Ma, and Chi]{Tong_2021}
Tong, T., Ma, C., and Chi, Y.
\newblock Low-rank matrix recovery with scaled subgradient methods: Fast and robust convergence without the condition number.
\newblock \emph{IEEE Transactions on Signal Processing}, 69:\penalty0 2396–2409, 2021{\natexlab{a}}.
\newblock ISSN 1941-0476.
\newblock \doi{10.1109/tsp.2021.3071560}.
\newblock URL \url{http://dx.doi.org/10.1109/TSP.2021.3071560}.

\bibitem[Tong et~al.(2021{\natexlab{b}})Tong, Ma, and Chi]{tong2021accelerating}
Tong, T., Ma, C., and Chi, Y.
\newblock Accelerating ill-conditioned low-rank matrix estimation via scaled gradient descent, 2021{\natexlab{b}}.

\bibitem[Tong et~al.(2022)Tong, Ma, Prater-Bennette, Tripp, and Chi]{tong2022scaling}
Tong, T., Ma, C., Prater-Bennette, A., Tripp, E., and Chi, Y.
\newblock Scaling and scalability: Provable nonconvex low-rank tensor estimation from incomplete measurements, 2022.

\bibitem[Valipour et~al.(2023)Valipour, Rezagholizadeh, Kobyzev, and Ghodsi]{valipour2023dylora}
Valipour, M., Rezagholizadeh, M., Kobyzev, I., and Ghodsi, A.
\newblock Dylora: Parameter efficient tuning of pre-trained models using dynamic search-free low-rank adaptation, 2023.

\bibitem[Wang et~al.(2018)Wang, Singh, Michael, Hill, Levy, and Bowman]{wang-etal-2018-glue}
Wang, A., Singh, A., Michael, J., Hill, F., Levy, O., and Bowman, S.
\newblock {GLUE}: A multi-task benchmark and analysis platform for natural language understanding.
\newblock In Linzen, T., Chrupa{\l}a, G., and Alishahi, A. (eds.), \emph{Proceedings of the 2018 {EMNLP} Workshop {B}lackbox{NLP}: Analyzing and Interpreting Neural Networks for {NLP}}, pp.\  353--355, Brussels, Belgium, November 2018. Association for Computational Linguistics.
\newblock \doi{10.18653/v1/W18-5446}.
\newblock URL \url{https://aclanthology.org/W18-5446}.

\bibitem[Yang(2020)]{yang2020scaling}
Yang, G.
\newblock Scaling limits of wide neural networks with weight sharing: Gaussian process behavior, gradient independence, and neural tangent kernel derivation, 2020.

\bibitem[Zhang et~al.(2023{\natexlab{a}})Zhang, Li, Chen, Jiang, Wang, and Qian]{Zhang2023IncreLoRAIP}
Zhang, F.~F., Li, L., Chen, J.-C., Jiang, Z., Wang, B., and Qian, Y.
\newblock Increlora: Incremental parameter allocation method for parameter-efficient fine-tuning.
\newblock \emph{ArXiv}, abs/2308.12043, 2023{\natexlab{a}}.
\newblock URL \url{https://api.semanticscholar.org/CorpusID:261076438}.

\bibitem[Zhang et~al.(2023{\natexlab{b}})Zhang, Chiu, and Zhang]{zhang2023fast}
Zhang, G., Chiu, H.-M., and Zhang, R.~Y.
\newblock Fast and minimax optimal estimation of low-rank matrices via non-convex gradient descent, 2023{\natexlab{b}}.

\bibitem[Zhang et~al.(2023{\natexlab{c}})Zhang, Fattahi, and Zhang]{zhang2023preconditioned}
Zhang, G., Fattahi, S., and Zhang, R.~Y.
\newblock Preconditioned gradient descent for overparameterized nonconvex burer--monteiro factorization with global optimality certification, 2023{\natexlab{c}}.

\bibitem[Zhang et~al.(2023{\natexlab{d}})Zhang, Chen, Bukharin, He, Cheng, Chen, and Zhao]{zhang2023adaptive}
Zhang, Q., Chen, M., Bukharin, A., He, P., Cheng, Y., Chen, W., and Zhao, T.
\newblock Adaptive budget allocation for parameter-efficient fine-tuning, 2023{\natexlab{d}}.

\bibitem[Zi et~al.(2023)Zi, Qi, Wang, Wang, Wong, and Zhang]{zi2023deltalora}
Zi, B., Qi, X., Wang, L., Wang, J., Wong, K.-F., and Zhang, L.
\newblock Delta-lora: Fine-tuning high-rank parameters with the delta of low-rank matrices, 2023.

\end{thebibliography}
\bibliographystyle{icml2024}

\newpage
\appendix
\onecolumn
\section{Details of Theorem \ref{stable_thm}}\label{stable_proof}
\subsection{Assumptions and Technical Lemmas}
\begin{definition}\label{thm3_def} (Stable Feature Learning) 
Consider any general LoRA layer $BAx$ with $B\in\mathbb{R}^{m\times r}$ and $A\in\mathbb{R}^{r\times n}$ being LoRA parameters. Denote $\Delta^t=B_tA_tx-B_{t-1}A_{t-1}x$ for fine-tuning step $t$. We say that LoRA model achieves Stable Feature Learning when $x,Ax,BAx\in\Theta(1)$ for all LoRA layers and $\Delta^t\in\Theta(1)$ for all fine-tuning step $t$.
\end{definition}
\begin{assumption}\label{assum}
We assume that the Adam gradient processing step satisfies $g_A^tx=\Theta(n)$ for all $t$ where $g_A^t$ is the normalized gradient of $A$ in $t$-th iteration.
\end{assumption}
\textbf{Explanation of Assumption \ref{assum}.} We adopt the same assumption as Assumption 1 in \cite{hayou2024lora} where the authors provide a proof for Adam with no momentum, i.e., for SignSGD method. This assumption should hold for general Adam variants as long as the processed gradient preserves the sign($x$) direction.

\begin{lemma}\label{thm3_lemma}
For any matrix $A\in\mathbb{R}^{m\times n}$, where $m$ being powers of $n$, such that $A^TA$ is invertible and $\gamma[A_{ij}]=c$  for all $(i,j)$, we have $\gamma[(A^TA)^{-1}]=-\gamma[\|a\|^2]$ with $a$ being any column of $A.$
\end{lemma}
\begin{proof}
First we note that $(A^TA)^{-1}=\text{adj}(A^TA)/\text{det}(A^TA)$ and $\text{det}(A^TA)=\Theta((mn^{2c})^n).$ Furthermore, by property of adjugate matrix,  \[\text{det}(\text{adj}(A^TA))=(\text{det}(A^TA))^{n-1}=\Theta((mn^{2c})^{n(n-1)}),\]
from which we deduce
\[\text{adj}(A^TA)=\Theta((mn^{2c})^{n-1}).\]
Therefore, 
\[(A^TA)^{-1}=\Theta((mn^{2c})^{-1}).\]
Note that $\|a\|^2=\Theta(mn^{2c}),$ we thus conclude $\gamma[(A^TA)^{-1}]=-\gamma[\|a\|^2]$, as desired.
\end{proof}
\subsection{Statement and Proof of Theorem \ref{stable_thm}}
Now, we state the formal version of our Theorem \ref{stable_thm} below,

\begin{theorem} (Stable Feature Learning (Formal)) Let $g_A$ and $g_B$ denote the processed gradient of $A$ and $B$ respectively. Consider LoRA parameters $A$ and $B$ trained with Adam scaled by preconditioner (\ref{formula}). Assume Assumption \ref{assum} is satisfied with Adam gradient processing and $g_A,g_B\in\Theta(1)$ after the gradient processing. Further assume $BAx$ has dimension of $\Theta(n).$  Then the LoRA model achieves  stable feature learning with $\eta=\Theta(1).$ While for unscaled Adam, $\eta_A=\Theta(n^{-1})$ and $\eta_B=\Theta(1)$ are required for stable feature learning.
\end{theorem}
\begin{proof}
 We consider Gaussian initialization with $A_{ij}\sim\mathcal{N}(0,\sigma_a^2)$ and $B_{ij}\sim\mathcal{N}(0,\sigma_b^2).$ Conventionally, we want $BA$ to be initialized as zero and $Ax$ does not explode with NN width, thus we proceed with $\sigma_a^2=\Theta(n^{-1})$ and $\sigma_b^2=0.$   We first decompose the LoRA increment as
 \begin{equation*}
 \begin{aligned}
 \Delta^t&=B_tA_tx-B_{t-1}A_{t-1}x\\
 &=(B_{t-1}-\eta g_B^{t-1}(A_{t-1}A_{t-1}^T)^{-1})(A_{t-1}-\eta(B_{t-1}^TB_{t-1})^{-1}g_A^{t-1})x-B_{t-1}A_{t-1}x\\
 &=-\eta B_{t-1}(B_{t-1}^TB_{t-1})^{-1}g_A^{t-1}x-\eta g_B^{t-1}(A_{t-1}A_{t-1}^T)^{-1}A_{t-1}x+\eta^2g_B^{t-1}(A_{t-1}A_{t-1}^T)^{-1}(B_{t-1}^TB_{t-1})^{-1}g_A^{t-1}x.
 \end{aligned}
 \end{equation*}
 We write
\[
\begin{cases}
\delta_t^1=\eta B_{t-1}(B_{t-1}^TB_{t-1})^{-1}g_A^{t-1}x,\\
\delta_t^2=\eta g_B^{t-1}(A_{t-1}A_{t-1}^T)^{-1}A_{t-1}x,\\
\delta_t^3=\eta^2g_B^{t-1}(A_{t-1}A_{t-1}^T)^{-1}(B_{t-1}^TB_{t-1})^{-1}g_A^{t-1}x.
\end{cases}
\]
Following Assumption \ref{assum}, we know $g_A^{t-1}x\in\Theta(n),$ thus having $\delta_t^1,\delta_t^2,B_{t-1}A_{t-1}x\in\Theta(1)$ equates to 
\begin{equation}\label{eq1}
\begin{cases}
\gamma[\eta]+\gamma[B_{t-1}]+\gamma[(B_{t-1}^TB_{t-1})^{-1}]+1=0,\\
\gamma[\eta]+\gamma[(A_{t-1}A_{t-1}^T)^{-1}]+\gamma[A_{t-1}x]=0,\\
\gamma[B_{t-1}]+\gamma[A_{t-1}x]=0.
\end{cases}
\end{equation}
For gradient update, we have 
\begin{equation*}
\begin{aligned}
B_t&=B_{t-1}-\eta g_B^{t-1}(A_{t-1}A_{t-1}^T)^{-1}\\
A_tx&=A_{t-1}x-\eta (B_{t-1}^TB_{t-1})^{-1}g_A^{t-1}x,
\end{aligned}
\end{equation*}
and thus
\begin{equation*}
\begin{aligned}
\gamma[B_t]&=\max(\gamma[B_{t-1}],\gamma[\eta]+\gamma[(A_{t-1}A_{t-1}^T)^{-1}])\\
\gamma[A_tx]&=\max(\gamma[A_{t-1}x],\gamma[\eta]+\gamma[(B_{t-1}^TB_{t-1})^{-1}]+1).
\end{aligned}
\end{equation*}
Note $A_1=A_0$ and thus $\gamma[A_1x]=\gamma[A_0x]=0.$ Furthermore, $\gamma[B_1]=\gamma[\eta]+\gamma[(A_0A_0^T)^{-1}].$ Since $\gamma[\|a_0\|_2^2]=0$ for any row $a_0$ of $A_0,$ $ \gamma[(A_0A_0^T)^{-1}]=0$ by Lemma \ref{thm3_lemma}. Therefore $\gamma[B_1]=0.$ Since $\gamma(\|b_1\|_2^2)=1$ for any column $b_1$ of $B_1$, $\gamma[(B_1^TB_1)^{-1}]=-1$ by Lemma \ref{thm3_lemma} and thus $\gamma[A_2x]=0$ by the above recursion.  Since $A_2=A_1-\eta (B_1^TB_1)^{-1}g_A^1=A_0-\Theta(n^{-1}), \gamma[\|a_2\|^2]=0$ for any row $a_2$ of $A_2$ and again by Lemma \ref{thm3_lemma} we know $\gamma[(A_2A_2^T)^{-1}]=0.$ Therefore $\gamma[B_2]=0.$ The recursion persists and we know $\gamma[B_t]=\gamma[A_tx]=0$ for all $t.$ Since $\gamma[(B_{t-1}^TB_{t-1})^{-1}]=-1$ and $\gamma[(A_{t-1}A_{t-1}^T)^{-1}]=0$, all equations in (\ref{eq1}) are satisfied. One can check that $\delta_t^3\in\Theta(1)$ and therefore stable feature learning is achieves with $\eta=\Theta(1).$ Theorem 1 in \cite{hayou2024lora} 
 shows that $\eta_A=\Theta(n^{-1})$ and $\eta_B=\Theta(1)$ are required for unpreconditioned LoRA training to achieve stable feature learning.
\subsection{Explanation of Different Learning Rates}
Here we note that learning rate $\eta=\Theta(1)$ is required for Theorem \ref{stable_thm} while learning rate $\eta=\Theta(n^{-1})$ is used for our toy example described in Section \ref{stable_feature}. This discrepancy arises from different settings being considered in these two regimes. Specifically, Theorem \ref{stable_thm} deals with vector output, i.e., we assume $BAx$ is of dimension $\Theta(n)$. This is core to our preconditioner since we then have $\gamma[(B_{t-1}^TB_{t-1})^{-1}]=-1$ from $\gamma[B_{t-1}]=0.$ For scalar output as considered in the toy example, when $\gamma[B_{t-1}]=0$, we will have $\gamma[(B_{t-1}^TB_{t-1})^{-1}]=0$ and thus fail to scale the statics correctly. Then one would wonder for scalar output, whether setting $\eta=\Theta(n^{-1})$ would be the correct choice as in the toy example. This is no longer true due to our Assumption \ref{assum}. In the toy example, we have $g_a^t=(f_t(x)-y)bx$ and ${g_a^t}^Tx=\Theta(n)$ is not guaranteed since it also scales with $f_t(x)$ and $b$. Instead, Theorem \ref{stable_thm} would hold for scalar output with $\eta_1=\Theta(1)$ and $\eta_t=\Theta(n^{-1})$ for $t>1$ which we do not include in the theorem statement for simplicity and one can deduce following our proof technique of Theorem \ref{stable_thm}. The takeaway is that for scalar output,  our preconditioner can still achieve stable feature learning with same order of magnitude learning rate for both $A$ and $B$ though one may need to tune learning rates across iterations, which is the current convention of learning rate scheduling.
 
\end{proof}

\section{Proof of Theorem \ref{main_thm}}\label{proof}
Note problem (\ref{main_prob}) is equivalent to the problem below up to a change of labels,
\begin{equation}\label{opt_prob}
\min_{A_i,B_i} {\|\sum_{i=1}^P C_iA_iB_i^T-Y\|_F^2},
\end{equation}
where $C_i\in\mathbb{R}^{n\times d}, A_i\in\mathbb{R}^{d\times r},B_i\in\mathbb{R}^{c\times r},Y\in\mathbb{R}^{n\times c}.$ Consider \textcolor{black}{$Y=\sum_{i=1}^P C_iA_\star^i{B_\star^i}^T$}. Denote $X_\star^i=A_\star^i{B_\star^i}^T=U_\star^i\Sigma_\star^i{V_\star^i}^T$ where $U_\star^i\Sigma_\star^i{V_\star^i}^T$ is the singular value decomposition of $X_\star^i$. Denote $F_\star^i=[A_\star^i, B_\star^i]^T.$ For any $F=[A,B]^T$, consider the following distance metric
\begin{equation}\label{dist_metric}
\text{dist}^2(F,F_\star^i)=\inf_{Q\in \textcolor{black}{GL(r)}}\left\|(AQ-A_\star^i){\Sigma_\star^i}^{1/2}\right\|_F^2+\left\|(BQ^{-T}-B_\star^i){\Sigma_\star^i}^{1/2}\right\|_F^2,
\end{equation}
where $GL(r)$ denotes the set of invertible matrix in $\mathbb{R}^{r\times r}$. Let $\sigma_r(\cdot)$ denote the $r$th largest singular value and $\kappa_i$ denote condition number of $X_\star^i$. Consider the following scaled GD step:
\[A_{t+1}^i=A_t^i-\eta C_i^T (\sum_j C_jA_t^j{B_t^j}^T-Y)B_t^i({B_t^i}^TB_t^i)^{-1},\]
\[B_{t+1}^i=B_t^i-\eta (C_i^T (\sum_j C_jA_t^j{B_t^j}^T-Y))^TA_t^i({A_t^i}^TA_t^i)^{-1}.\]


Before we begin the main proof, we need the following partial Frobenius norm which has been introduced in Section A.3 in \cite{tong2021accelerating} with some important properties studied there.
\begin{definition} (Partial Frobenius norm) For any matrix $X$, its partial Frobenius norm of order $r$ is given by $l_2$ norm of vectors composed by its top-$r$ singular values, 
\[\|X\|_{F,r}=\sqrt{\sum_{i=1}^r \sigma_i^2(X)}.\]
\end{definition}

Now we start the proof by first proving some useful lemmas.
\begin{lemma}\label{lemma1}
Under Assumption \ref{main_assum}, let $F_0^i=[A_0^i,B_0^i]^T$, then the extented \textcolor{black} spectral initialization in Definition
\ref{spectral_init} satisfies
\[\text{dist}(F_0^i,F_\star^i)\leq 10\delta_{2r}^i \sqrt{r} \kappa_i\sigma_r(X_\star^i).\]
\end{lemma}
\begin{proof}
According to Lemma 11 in \cite{tong2021accelerating}, since \textcolor{black}{$A_0^i{B_0^i}^T-X_\star^i$} has rank at most $2r$, 
\begin{equation*}
\begin{aligned}
\text{dist}(F_0^i,F_\star^i)&\leq \sqrt{\sqrt{2}+1} \|A_0^i{B_0^i}^T-X_\star^i\|_F,\\
&\leq \sqrt{2(\sqrt{2}+1)}\|A_0^i{B_0^i}^T-X_\star^i\|_{\textcolor{black}{F,r}}.
\end{aligned}
\end{equation*}
Since $A_0^i{B_0^i}^T$ is the best rank $r$ approximation of 
$C_i^TY=\sum_{j=1}^P C_i^TC_jX_\star^j$.
Then 
\begin{equation*}
\begin{aligned}
\|A_0^i{B_0^i}^T-X_\star^i\|_{F,r}&\leq \|\sum_{j=1}^P C_i^TC_jX_\star^j-A_0^i{B_0^i}^T\|_{F,r}+\|\sum_{j=1}^P C_i^TC_jX_\star^j-X_\star^i\|_{F,r},\\
&\leq 2\|(C_i^T C_i-I)X_\star^i+\sum_{j\neq i}C_i^TC_jX_\star^j\|_{F,r},\\
&\leq \textcolor{black}{2\delta_{2r}^i \|X_\star^i\|_F}+2\sum_{j\neq i}\|C_i^TC_jX_\star^j\|_{F},
\end{aligned}
\end{equation*}
where the last inequality follows Lemma 15 and inequality (50) in  \cite{tong2021accelerating}. Therefore, 
\begin{equation*}
\begin{aligned}\text{dist}(F_0^i,F_\star^i)&\leq 5\delta_{2r}^i \|X_\star^i\|_F+5\sum_{j\neq i}\textcolor{black}{\|C_i^T C_jX_\star^j\|_{F}},\\
&\leq 10\delta_{2r}^i \|X_\star^i\|_F \leq 10\delta_{2r}^i \sqrt{r}\kappa_i \sigma_r(X_\star^i).
\end{aligned}
\end{equation*}
\end{proof}

\begin{lemma}\label{lemma2} (Contraction) Under assumption \ref{main_assum} with $\delta_{2r}^i\leq 0.01$. If the $t$-th iterate satisfies $\text{dist}(F_t^i,F_\star^i)\leq 0.1\sigma_r(X_\star^i)$ where $F_t^i=[A_t^i,B_t^i]^T$, then $\|A_t^i{B_t^i}^T-X_\star^i\|_F\leq 1.5\text{dist}(F_t^i,F_\star^i)$. In addition, if the step size $0<\eta\leq 2/3$, then the $(t+1)$-th iteration $F_{t+1}^i$ satisfies
\[\max(\text{dist}(F_{t+1}^i, F_\star^i))\leq (1-0.5\eta)\max(\text{dist}(F_{t}^i, F_\star^i)).\]
\end{lemma}
\begin{proof}
We first show $\|A_t^i{B_t^i}^T-X_\star^i\|_F\leq 1.5\text{dist}(F_t^i, F_\star^i).$ According to Lemma 9 in \cite{tong2021accelerating}, we know $Q_t^i,$  the optimal alignment matrix between $F_t^i$ and $F_\star^i$ , i.e., the optimal value of problem (\ref{opt_prob}) with $F$ replaced by $F_t^i$ is attained at $Q_t^i$, exists, denote $A^i=A_t^i Q_t^i, B^i=B_t^i{Q_t^i}^{-T},\triangle_A^i=A^i-A_\star^i,\triangle_B^i=B^i-B_\star^i.$ By derivation (45) in \cite{tong2021accelerating}, we further know for $\epsilon=0.1,$
\[\|\triangle_A^i {\Sigma_\star^i}^{-1/2}\|_2\lor \|\triangle_B^i {\Sigma_\star^i}^{-1/2}\|_2\leq \epsilon,\]
where $\lor$ denotes maximum.  Note
\begin{equation}\label{general_point_contraction}
\begin{aligned}
\|A_t^i{B_t^i}^T-X_\star^i\|_F&=\|A^i{B^i}^T-X_\star^i\|_F\\
&=\|\triangle_A^i{B^i}^T+A_\star^i{\triangle_B^i}^T\|_F\\
&=\|\triangle_A^i{\triangle_B^i}^T+\triangle_A^i{B_\star^i}^T+A_\star^i{\triangle_B^i}^T\|_F\\
&\leq \|\triangle_A^i{\triangle_B^i}^T\|_F+\|\triangle_A^i{B_\star^i}^T\|_F+\|A_\star^i{\triangle_B^i}^T\|_F\\
&=\|\triangle_A^i {\Sigma_\star^i}^{1/2}\|_F+\|\triangle_B^i{\Sigma_\star^i}^{1/2}\|_F+\|\triangle_A^i{\triangle_B^i}^T\|_F\\
&\leq \|\triangle_A^i {\Sigma_\star^i}^{1/2}\|_F+\|\triangle_B^i{\Sigma_\star^i}^{1/2}\|_F+\\
&\qquad \frac{1}{2}(\|\triangle_A^i {\Sigma_\star^i}^{-1/2}\|_2\lor\|\triangle_B^i {\Sigma_\star^i}^{-1/2}\|_2)(\|\triangle_A^i {\Sigma_\star^i}^{1/2}\|_F+\|\triangle_B^i {\Sigma_\star^i}^{1/2}\|_F)\\
&\leq (1+\frac{\epsilon}{2})(\|\triangle_A^i {\Sigma_\star^i}^{1/2}\|_F+\|\triangle_B^i {\Sigma_\star^i}^{1/2}\|_F)\\
&\leq (1+\frac{\epsilon}{2})\sqrt{2}\text{dist}(F_t^i,F_\star^i)\leq 1.5\text{dist}(F_t^i,F_\star^i).
\end{aligned}
\end{equation}
Note the second last inequality follows from $\text{dist}(F_t^i,F_\star^i)=\sqrt{\|\triangle_A^i {\Sigma_\star^i}^{1/2}\|_F^2+\|\triangle_B^i{\Sigma_\star^i}^{1/2}\|_F^2}$. We then proceed to show the contraction of distance. By definition, 
\[\text{dist}^2(F_{t+1}^i,F_\star^i)\leq \|(A_{t+1}^iQ_t^i-A_\star^i){\Sigma_\star^i}^{1/2}\|_F^2+\|(B_{t+1}^i{Q_t^i}^{-T}-B_\star^i){\Sigma_\star^i}^{1/2}\|_F^2.\]
Substitute the update rule for $L_{t+1}^i$ we get
\begin{equation*}
\begin{aligned}
(A_{t+1}^iQ_t^i-A_\star^i){\Sigma_\star^i}^{1/2}&=(A_t^iQ_t^i-\eta C_i^T (\sum_j C_jA_t^j{B_t^j}^T-Y)B_t^i({B_t^i}^TB_t^i)^{-1}Q_t^i-A_\star^i){\Sigma_\star^i}^{1/2}\\
&=(\triangle_A^i-\eta C_i^T (\sum_j C_jA_t^j{B_t^j}^T-Y)B_t^i({B_t^i}^TB_t^i)^{-1}Q_t^i){\Sigma_\star^i}^{1/2}\\
&=(\triangle_A^i-\eta C_i^T (\sum_j C_jA_t^j{B_t^j}^T-Y)B^i({B^i}^TB^i)^{-1}){\Sigma_\star^i}^{1/2}\\
&=(\triangle_A^i-\eta C_i^TC_i(A^i{B^i}^T-X_\star^i)B^i({B^i}^TB^i)^{-1}-\eta C B^i({B^i}^TB^i)^{-1}){\Sigma_\star^i}^{1/2}\\
& \text{where }C=\sum_{j\neq i}C_i^TC_j (A_t^j{B_t^j}^T-X_\star^j),\\
&=(\triangle_A^i-\eta (A^i{B^i}^T-X_\star^i)B^i({B^i}^TB^i)^{-1}-\eta(C_i^TC_i-I)(A^i{B^i}^T-X_\star^i)B^i({B^i}^TB^i)^{-1}\\&\qquad -\eta C B^i({B^i}^TB^i)^{-1}){\Sigma_\star^i}^{1/2}\\
&\text{since } A^i{B^i}^T-X_\star^i= \triangle_A^i {B^i}^T+A_\star^i{\triangle_B^i}^T,\\
&=(\triangle_A^i-\eta \triangle_A^i-\eta A_\star^i{\triangle_B^i}^TB^i({B^i}^TB^i)^{-1}-\eta(C_i^TC_i-I)(A^i{B^i}^T-X_\star^i)B^i({B^i}^TB^i)^{-1}\\&\qquad -\eta C B^i({B^i}^TB^i)^{-1}){\Sigma_\star^i}^{1/2}\\
&=(1-\eta)\triangle_A^i{\Sigma_\star^i}^{1/2}-\eta A_\star^i{\triangle_B^i}^TB^i({B^i}^TB^i)^{-1}{\Sigma_\star^i}^{1/2}-\\&\qquad \eta(C_i^TC_i-I)(A^i{B^i}^T-X_\star^i)B^i({B^i}^TB^i)^{-1}{\Sigma_\star^i}^{1/2}-\eta C B^i({B^i}^TB^i)^{-1}{\Sigma_\star^i}^{1/2}.
\end{aligned}
\end{equation*}
Therefore, 
\begin{equation}\label{general_left}
\begin{aligned}
\|(A_{t+1}^iQ_t^i-A_\star^i){\Sigma_\star^i}^{1/2}\|_F^2&=\|(1-\eta)\triangle_A^i{\Sigma_\star^i}^{1/2}-\eta A_\star^i{\triangle_B^i}^TB^i({B^i}^TB^i)^{-1}{\Sigma_\star^i}^{1/2}\|_F^2\\
&\qquad +\eta^2\|(C_i^TC_i-I)(A^i{B^i}^T-X_\star^i)B^i({B^i}^TB^i)^{-1}{\Sigma_\star^i}^{1/2}+C B^i({B^i}^TB^i)^{-1}{\Sigma_\star^i}^{1/2}\|_F^2\\
&\qquad -2\eta\text{tr}(((1-\eta){\Sigma_\star^i}^{1/2}{\triangle_A^i}^T-\eta{\Sigma_\star^i}^{1/2}({B^i}^TB^i)^{-1}{B^i}^T\triangle_B^i{A_\star^i}^T)\\&\qquad \quad((C_i^TC_i-I)(A^i{B^i}^T-X_\star^i)B^i({B^i}^TB^i)^{-1}{\Sigma_\star^i}^{1/2}+C B^i({B^i}^TB^i)^{-1}{\Sigma_\star^i}^{1/2})).
\end{aligned}
\end{equation}
Follow derivation (46) in \cite{tong2021accelerating}, we can bound
\begin{equation}\label{general_bb1}
\begin{aligned}
&\|(1-\eta)\triangle_A^i{\Sigma_\star^i}^{1/2}-\eta A_\star^i{\triangle_B^i}^TB^i({B^i}^TB^i)^{-1}{\Sigma_\star^i}^{1/2}\|_F^2\\&\leq \left((1-\eta)^2+\frac{2\epsilon\eta(1-\eta)}{1-\epsilon}\right)\|\triangle_A^i{\Sigma_\star^i}^{1/2}\|_F^2+\frac{\eta^2(2\epsilon+\epsilon^2)}{(1-\epsilon)^2}\|\triangle_B^i{\Sigma_\star^i}^{1/2}\|_F^2.
\end{aligned}
\end{equation}
Next, we want to bound
\begin{equation*}
\begin{aligned}
\|(C_i^TC_i-I)(A^i{B^i}^T-X_\star^i)B^i({B^i}^TB^i)^{-1}{\Sigma_\star^i}^{1/2}+C B^i({B^i}^TB^i)^{-1}{\Sigma_\star^i}^{1/2}\|_F^2\\
=\|(C_i^TC_i-I)(A^i{B^i}^T-X_\star^i)B^i({B^i}^TB^i)^{-1}{\Sigma_\star^i}^{1/2}\|_F^2+\|C B^i({B^i}^TB^i)^{-1}{\Sigma_\star^i}^{1/2}\|_F^2\\
+2\text{tr}(((C_i^TC_i-I)(A^i{B^i}^T-X_\star^i)B^i({B^i}^TB^i)^{-1}{\Sigma_\star^i}^{1/2})^TC B^i({B^i}^TB^i)^{-1}{\Sigma_\star^i}^{1/2}).
\end{aligned}
\end{equation*}
By the bound for $\mathfrak{S}_4$ in Lemma 1 in  \cite{tong2021accelerating}, we can bound 
\[\|(C_i^TC_i-I)(A^i{B^i}^T-X_\star^i)B^i({B^i}^TB^i)^{-1}{\Sigma_\star^i}^{1/2}\|_F^2\leq \frac{{\delta_{2r}^i}^2(2+\epsilon)^2}{2(1-\epsilon)^2}(\|\triangle_A^i{\Sigma_\star^i}^{1/2}\|_F^2+\|\triangle_B^i{\Sigma_\star^i}^{1/2}\|_F^2)\]
We now proceed to bound $\|C B^i({B^i}^TB^i)^{-1}{\Sigma_\star^i}^{1/2}\|_F^2$. According to Lemma 9 in \cite{tong2021accelerating}, we know \textcolor{black}{$Q_t^j$ exists}, the optimal alignment matrix between $F_t^j$ and $F_\star^j$ exists. Denote $A^j=A_t^jQ_t^j,B^j=B_t^j{Q_t^j}^{-T},\triangle_A^j=A^j-A_\star^j,\triangle_B^j=B^j-B_\star^j$. Since
\begin{equation}\label{general_eq1}
\begin{aligned}
&\|C_i^TC_j (A_t^j{B_t^j}^T-X_\star^j)B^i({B^i}^TB^i)^{-1}{\Sigma_\star^i}^{1/2}\|_F\\
&=\|C_i^TC_j (\triangle_A^j{\triangle_B^j}^T+\triangle_A^j{B_\star^j}^T+A_\star^j{\triangle_B^j}^T)B^i({B^i}^TB^i)^{-1}{\Sigma_\star^i}^{1/2}\|_F\\
&\leq \textcolor{black}{\|C_i^TC_j\|_2}(\|\triangle_A^j{\triangle_B^j}^T\|_F+\|\triangle_A^j{B_\star^j}^T\|_F+\|A_\star^j{\triangle_B^j}^T\|_F)\|B^i({B^i}^TB^i)^{-1}{\Sigma_\star^i}^{1/2}\|_2\\
&\text{by Lemma 12 in \cite{tong2021accelerating} and (\ref{general_point_contraction}),}\\
&\leq \frac{(2+\epsilon)\textcolor{black}{\|C_i^TC_j\|_2}}{2(1-\epsilon)}(\|\triangle_A^j{\Sigma_\star^j}^{1/2}\|_F+\|\triangle_B^j{\Sigma_\star^j}^{1/2}\|_F).
\end{aligned}
\end{equation}
Thus
\begin{equation*}\begin{aligned}\|C B^i({B^i}^TB^i)^{-1}{\Sigma_\star^i}^{1/2}\|_F^2&\leq (P-1)\sum_{j\neq i} \|C_i^TC_j (A_t^j{B_t^j}^T-X_\star^j)B^i({B^i}^TB^i)^{-1}{\Sigma_\star^i}^{1/2}\|_F^2\\
&\leq (P-1)\sum_{j\neq i} \frac{(2+\epsilon)^2\|C_i^TC_j\|_2^2}{2(1-\epsilon)^2}(\|\triangle_A^j{\Sigma_\star^j}^{1/2}\|_F^2+\|\triangle_B^j{\Sigma_\star^j}^{1/2}\|_F^2).
\end{aligned}\end{equation*}
Next we bound
\begin{equation*}
\begin{aligned}
&|\text{tr}(((C_i^TC_i-I)(A^i{B^i}^T-X_\star^i)B^i({B^i}^TB^i)^{-1}{\Sigma_\star^i}^{1/2})^TC B^i({B^i}^TB^i)^{-1}{\Sigma_\star^i}^{1/2})|\\
&\leq \|B^i({B^i}^TB^i)^{-1}{\Sigma_\star^i}^{1/2}\|_2^2|\text{tr}(((C_i^TC_i-I)(A^i{B^i}^T-X_\star^i))^TC)|\\
&\text{ by Lemma 12 in \cite{tong2021accelerating},}\\
&\leq \frac{1}{(1-\epsilon)^2}|\text{tr}(((C_i^TC_i-I)(A^i{B^i}^T-X_\star^i))^TC)|\\
&\text{by Lemma 17 in \cite{tong2021accelerating}, since $C_i$ is $2r$-RIP,}\\
&\leq \frac{{\delta_{2r}^i}}{(1-\epsilon)^2}\|C\|_F\|A^i{B^i}^T-X_\star^i\|_F\\
&\text{by derivation in (\ref{general_eq1}),}\\
&\leq \sum_{j\neq  i} \frac{(1+\frac{\epsilon}{2})^2\delta_{2r}^i}{(1-\epsilon)^2}\textcolor{black}{\|C_i^TC_j\|_2}(\|\triangle_A^j{\Sigma_\star^j}^{1/2}\|_F+\|\triangle_B^j{\Sigma_\star^j}^{1/2}\|_F)(\|\triangle_A^i{\Sigma_\star^i}^{1/2}\|_F+\|\triangle_B^i{\Sigma_\star^i}^{1/2}\|_F).
\end{aligned}
\end{equation*}
To summarize,
\begin{equation}\label{general_bb2}
\begin{aligned}
&\|(C_i^TC_i-I)(A^i{B^i}^T-X_\star^i)B^i({B^i}^TB^i)^{-1}{\Sigma_\star^i}^{1/2}+C B^i({B^i}^TB^i)^{-1}{\Sigma_\star^i}^{1/2}\|_F^2\\
&\leq \frac{{\delta_{2r}^i}^2(2+\epsilon)^2}{2(1-\epsilon)^2}(\|\triangle_A^i{\Sigma_\star^i}^{1/2}\|_F^2+\|\triangle_B^i{\Sigma_\star^i}^{1/2}\|_F^2)+(P-1)\sum_{j\neq i}\frac{\textcolor{black}{\|C_i^TC_j\|_2^2}(2+\epsilon)^2}{2(1-\epsilon)^2}(\|\triangle_A^j{\Sigma_\star^j}^{1/2}\|_F^2+\|\triangle_B^j{\Sigma_\star^j}^{1/2}\|_F^2)\\
&\qquad +\sum_{j\neq i}\frac{2(1+\frac{\epsilon}{2})^2\sigma_{2r}^i}{(1-\epsilon)^2}\textcolor{black}{\|C_i^TC_j\|_2}(\|\triangle_A^j{\Sigma_\star^j}^{1/2}\|_F+\|\triangle_B^j{\Sigma_\star^j}^{1/2}\|_F)(\|\triangle_A^i{\Sigma_\star^i}^{1/2}\|_F+\|\triangle_B^i{\Sigma_\star^i}^{1/2}\|_F).
\end{aligned}
\end{equation}
Finally, we move on to bound
\begin{align}
&|\text{tr}(((1-\eta){\Sigma_\star^i}^{1/2}{\triangle_A^i}^T-\eta{\Sigma_\star^i}^{1/2}({B^i}^TB^i)^{-1}{B^i}^T\triangle_B^i{A_\star^i}^T)((C_i^TC_i-I)(A^i{B^i}^T-X_\star^i)B^i({B^i}^TB^i)^{-1}{\Sigma_\star^i}^{1/2} \label{general_eq2}\\
&+C B^i({B^i}^TB^i)^{-1}{\Sigma_\star^i}^{1/2}))| \nonumber\\
&\leq|\text{tr}((1-\eta){\Sigma_\star^i}^{1/2}{\triangle_A^i}^T(C_i^TC_i-I)(A^i{B^i}^T-X_\star^i)B^i({B^i}^TB^i)^{-1}{\Sigma_\star^i}^{1/2})| \nonumber\\
&\qquad +|\text{tr}(\eta {\Sigma_\star^i}^{1/2}({B^i}^TB^i)^{-1}{B^i}^T\triangle_B^i{A_\star^i}^T(C_i^TC_i-I)(A^i{B^i}^T-X_\star^i)B^i({B^i}^TB^i)^{-1}{\Sigma_\star^i}^{1/2})| \nonumber\\
&\qquad +|\text{tr}((1-\eta){\Sigma_\star^i}^{1/2}{\triangle_A^i}^TCB^i({B^i}^TB^i)^{-1}{\Sigma_\star^i}^{1/2})|+|\text{tr}(\eta{\Sigma_\star^i}^{1/2}({B^i}^TB^i)^{-1}{B^i}^T\triangle_B^i{A_\star^i}^TCB^i({B^i}^TB^i)^{-1}{\Sigma_\star^i}^{1/2})|. \nonumber
\end{align}
First notice by the bound for  $|\mathfrak{S}_2|$  in Lemma 1 in \cite{tong2021accelerating},
\begin{equation}\label{general_b1}
\begin{aligned}
&|\text{tr}((1-\eta){\Sigma_\star^i}^{1/2}{\triangle_A^i}^T(C_i^TC_i-I)(A^i{B^i}^T-X_\star^i)B^i({B^i}^TB^i)^{-1}{\Sigma_\star^i}^{1/2})|\\
&\leq \frac{(1-\eta)\delta_{2r}^i(2+\epsilon)}{2(1-\epsilon)}(\frac{3}{2}\|\triangle_A^i{\Sigma_\star^i}^{1/2}\|_F^2+\frac{1}{2}\|\triangle_B^i{\Sigma_\star^i}^{1/2}\|_F^2).
\end{aligned}
\end{equation}
Similarly, by the bound for $|\mathfrak{S}_3|$  in Lemma 1 in \cite{tong2021accelerating},
\begin{equation}\label{general_b2}
\begin{aligned}
&|\text{tr}(\eta {\Sigma_\star^i}^{1/2}({B^i}^TB^i)^{-1}{B^i}^T\triangle_B^i{A_\star^i}^T(C_i^TC_i-I)(A^i{B^i}^T-X_\star^i)B^i({B^i}^TB^i)^{-1}{\Sigma_\star^i}^{1/2})|\\
&\leq \frac{\eta \delta_{2r}^i (2+\epsilon)}{2(1-\epsilon)^2}(\frac{3}{2}\|\triangle_B^i{\Sigma_\star^i}^{1/2}\|_F^2+\frac{1}{2}\|\triangle_A^i{\Sigma_\star^i}^{1/2}\|_F^2).
\end{aligned}
\end{equation}
Next, consider bounding
\begin{equation}\label{general_b3}
\begin{aligned}
&|\text{tr}((1-\eta){\Sigma_\star^i}^{1/2}{\triangle_A^i}^TCB^i({B^i}^TB^i)^{-1}{\Sigma_\star^i}^{1/2})|\\
&\leq (1-\eta)\sum_{j\neq i}|\text{tr}({\Sigma_\star^i}^{1/2}{\triangle_A^i}^TC_i^TC_j (A_t^j{B_t^j}^T-X_\star^j)B^i({B^i}^TB^i)^{-1}{\Sigma_\star^i}^{1/2})|\\
&\leq (1-\eta)\sum_{j\neq i}\textcolor{black}{\|C_i^TC_j\|_2}\|A_t^j{B_t^j}^T-X_\star^j\|_F\|B^i({B^i}^TB^i)^{-1}{\Sigma_\star^i}{\triangle_A^i}^T\|_F\\
&\text{by Lemma 12 in \cite{tong2021accelerating},}\\
&\leq \frac{1-\eta}{1-\epsilon}\sum_{j\neq i}\|C_i^TC_j\|_2\|A_t^j{B_t^j}^T-X_\star^j\|_F\|\triangle_A^i {\Sigma_\star^i}^{1/2}\|_F\\
&\text{by derivation in (\ref{general_point_contraction}),}\\
&\leq \frac{(2+\epsilon)(1-\eta)}{2(1-\epsilon)}\sum_{j\neq i} \|C_i^TC_j\|_2(\|\triangle_A^j {\Sigma_\star^j}^{1/2}\|_F+\|\triangle_B^j {\Sigma_\star^j}^{1/2}\|_F)\|\triangle_A^i {\Sigma_\star^i}^{1/2}\|_F,
\end{aligned}
\end{equation}
and
\begin{equation}\label{general_b4}
\begin{aligned}
&|\text{tr}(\eta{\Sigma_\star^i}^{1/2}({B^i}^TB^i)^{-1}{B^i}^T\triangle_B^i{A_\star^i}^TCB^i({B^i}^TB^i)^{-1}{\Sigma_\star^i}^{1/2})|\\
&=\eta|\text{tr}(\sum_{j\neq i}C_i^TC_j (A_t^j{B_t^j}^T-X_\star^j)B^i({B^i}^TB^i)^{-1}{\Sigma_\star^i}({B^i}^TB^i)^{-1}{B^i}^T\triangle_B^i{A_\star^i}^T)|\\
& \leq \eta\sum_{j\neq i}\|C_i^TC_j\|_2\|A_t^j{B_t^j}^T-X_\star^j\|_F\|B^i({B^i}^TB^i)^{-1}{\Sigma_\star^i}({B^i}^TB^i)^{-1}{B^i}^T\triangle_B^i{A_\star^i}^T\|_F\\
&\leq \eta\sum_{j\neq i}\|C_i^TC_j\|_2\|A_t^j{B_t^j}^T-X_\star^j\|_F\|B^i({B^i}^TB^i)^{-1}{\Sigma_\star^i}^{1/2}\|_2^2\|\triangle_B^i{L_\star^i}^T\|_F\\
&\text{by Lemma 12 in \cite{tong2021accelerating},}\\
&\leq \frac{\eta}{(1-\epsilon)^2}\sum_{j\neq i}\|C_i^TC_j\|_2\|A_t^j{B_t^j}^T-X_\star^j\|_F\|\triangle_B^i{A_\star^i}^T\|_F\\
&\text{by derivation in (\ref{general_point_contraction}),}\\
&\leq \frac{\eta(2+\epsilon)}{2(1-\epsilon)^2}\sum_{j\neq i}\|C_i^TC_j\|_2(\|\triangle_A^j {\Sigma_\star^j}^{1/2}\|_F+\|\triangle_B^j {\Sigma_\star^j}^{1/2}\|_F)\|\triangle_B^i{A_\star^i}^T\|_F.
\end{aligned}
\end{equation}
Combine (\ref{general_b1}), (\ref{general_b2}), (\ref{general_b3}), (\ref{general_b4}) to bound (\ref{general_eq2})
\begin{equation}\label{general_bb3}
\begin{aligned}
&|\text{tr}(((1-\eta){\Sigma_\star^i}^{1/2}{\triangle_A^i}^T-\eta{\Sigma_\star^i}^{1/2}({B^i}^TB^i)^{-1}{B^i}^T\triangle_B^i{A_\star^i}^T)((C_i^TC_i-I)(A^i{B^i}^T-X_\star^i)B^i({B^i}^TB^i)^{-1}{\Sigma_\star^i}^{1/2} \\
&+C B^i({B^i}^TB^i)^{-1}{\Sigma_\star^i}^{1/2}))| \\
&\leq \frac{(1-\eta)\delta_{2r}^i(2+\epsilon)}{2(1-\epsilon)}(\frac{3}{2}\|\triangle_A^i{\Sigma_\star^i}^{1/2}\|_F^2+\frac{1}{2}\|\triangle_B^i{\Sigma_\star^i}^{1/2}\|_F^2)+\frac{\eta \delta_{2r}^i (2+\epsilon)}{2(1-\epsilon)^2}(\frac{3}{2}\|\triangle_B^i{\Sigma_\star^i}^{1/2}\|_F^2+\frac{1}{2}\|\triangle_A^i{\Sigma_\star^i}^{1/2}\|_F^2)\\
&+\frac{(2+\epsilon)(1-\eta)}{2(1-\epsilon)}\sum_{j\neq i}\|C_i^TC_j\|_2(\|\triangle_A^j {\Sigma_\star^j}^{1/2}\|_F+\|\triangle_B^j {\Sigma_\star^j}^{1/2}\|_F)\|\triangle_A^i {\Sigma_\star^i}^{1/2}\|_F\\
&+\frac{\eta(2+\epsilon)}{2(1-\epsilon)^2}\sum_{j\neq i}\|C_i^TC_j\|_2(\|\triangle_A^j {\Sigma_\star^j}^{1/2}\|_F+\|\triangle_B^j {\Sigma_\star^j}^{1/2}\|_F)\|\triangle_B^i{L_\star^i}^T\|_F\\
&\leq \frac{(1-\eta)\delta_{2r}^i(2+\epsilon)}{2(1-\epsilon)}(\frac{3}{2}\|\triangle_A^i{\Sigma_\star^i}^{1/2}\|_F^2+\frac{1}{2}\|\triangle_B^i{\Sigma_\star^i}^{1/2}\|_F^2)+\frac{\eta \delta_{2r}^i (2+\epsilon)}{2(1-\epsilon)^2}(\frac{3}{2}\|\triangle_B^i{\Sigma_\star^i}^{1/2}\|_F^2+\frac{1}{2}\|\triangle_A^i{\Sigma_\star^i}^{1/2}\|_F^2)\\
&+\frac{(2+\epsilon)(1-\eta)}{2(1-\epsilon)}\sum_{j\neq i}\|C_i^TC_j\|_2(\frac{1}{2}\|\triangle_A^j {\Sigma_\star^j}^{1/2}\|_F^2+\frac{1}{2}\|\triangle_B^j {\Sigma_\star^j}^{1/2}\|_F^2+\|\triangle_A^i{\Sigma_\star^i}^{1/2}\|_F^2)\\
&+\frac{\eta(2+\epsilon)}{2(1-\epsilon)^2}\sum_{j\neq i}\|C_i^TC_j\|_2(\frac{1}{2}\|\triangle_A^j {\Sigma_\star^j}^{1/2}\|_F^2+\frac{1}{2}\|\triangle_B^j {\Sigma_\star^j}^{1/2}\|_F^2+\|\triangle_B^i{L_\star^i}^T\|_F^2).
\end{aligned}
\end{equation}
Substituting bounds (\ref{general_bb1}), (\ref{general_bb2}), (\ref{general_bb3}), we derive bound for (\ref{general_left}) as
\begin{equation*}
\begin{aligned}
\|(A_{t+1}^iQ_t^i-A_\star^i){\Sigma_\star^i}^{1/2}\|_F^2&\leq ((1-\eta)^2+\frac{2\epsilon\eta(1-\eta)}{1-\epsilon})\|\triangle_A^i{\Sigma_\star^i}^{1/2}\|_F^2+\frac{\eta^2(2\epsilon+\epsilon^2)}{(1-\epsilon)^2}\|\triangle_B^i{\Sigma_\star^i}^{1/2}\|_F^2\\
&\qquad + \eta^2(\frac{{\delta_{2r}^i}^2(2+\epsilon)^2}{2(1-\epsilon)^2}(\|\triangle_A^i{\Sigma_\star^i}^{1/2}\|_F^2+\|\triangle_B^i{\Sigma_\star^i}^{1/2}\|_F^2)\\
&\qquad +(P-1)\sum_{j\neq i}\frac{\textcolor{black}{\|C_i^TC_j\|_2^2}(2+\epsilon)^2}{2(1-\epsilon)^2}(\|\triangle_A^j{\Sigma_\star^j}^{1/2}\|_F^2 +\|\triangle_B^j{\Sigma_\star^j}^{1/2}\|_F^2) \\
&\qquad +\sum_{j\neq i}\frac{2(1+\frac{\epsilon}{2})^2\delta_{2r}^i}{(1-\epsilon)^2}\textcolor{black}{\|C_i^TC_j\|}(\|\triangle_A^j{\Sigma_\star^j}^{1/2}\|_F+\|\triangle_B^j{\Sigma_\star^j}^{1/2}\|_F)(\|\triangle_A^i{\Sigma_\star^i}^{1/2}\|_F\\
\end{aligned}
\end{equation*}
\begin{equation*}
\begin{aligned}
&\qquad +\|\triangle_B^i{\Sigma_\star^i}^{1/2}\|_F))+2\eta(\frac{(1-\eta)\delta_{2r}^i(2+\epsilon)}{2(1-\epsilon)}(\frac{3}{2}\|\triangle_A^i{\Sigma_\star^i}^{1/2}\|_F^2+\frac{1}{2}\|\triangle_B^i{\Sigma_\star^i}^{1/2}\|_F^2)\\
&\qquad +\frac{\eta \delta_{2r}^i (2+\epsilon)}{2(1-\epsilon)^2}(\frac{3}{2}\|\triangle_B^i{\Sigma_\star^i}^{1/2}\|_F^2+\frac{1}{2}\|\triangle_A^i{\Sigma_\star^i}^{1/2}\|_F^2)\\
&\qquad +\frac{(2+\epsilon)(1-\eta)}{2(1-\epsilon)}\sum_{j\neq i}\|C_i^TC_j\|_2(\frac{1}{2}\|\triangle_A^j {\Sigma_\star^j}^{1/2}\|_F^2 +\frac{1}{2}\|\triangle_B^j {\Sigma_\star^j}^{1/2}\|_F^2+\|\triangle_A^i{\Sigma_\star^i}^{1/2}\|_F^2)\\
&\qquad +\frac{\eta(2+\epsilon)}{2(1-\epsilon)^2}\sum_{j\neq i}\|C_i^TC_j\|_2(\frac{1}{2}\|\triangle_A^j {\Sigma_\star^j}^{1/2}\|_F^2+\frac{1}{2}\|\triangle_B^j {\Sigma_\star^j}^{1/2}\|_F^2+\|\triangle_B^i{A_\star^i}^T\|_F^2)).
\end{aligned}
\end{equation*}
We similarly derive bound for $\|(B_{t+1}^i{Q_t^i}^{-T}-B_\star^i){\Sigma_\star^i}^{1/2}\|_F^2$. Since $\text{dist}^2(F_t^i,F_\star^i)=\text{tr}(\triangle_A^i\Sigma_\star^i{\triangle_A^i}^T)+\text{tr}(\triangle_B^i\Sigma_\star^i{\triangle_B^i}^T)$, we thus get
\begin{equation*}
\begin{aligned}
&\|(A_{t+1}^iQ_t^i-A_\star^i){\Sigma_\star^i}^{1/2}\|_F^2+\|(B_{t+1}^i{Q_t^i}^{-T}-B_\star^i){\Sigma_\star^i}^{1/2}\|_F^2\\
&\leq ((1-\eta)^2+\frac{2\epsilon\eta(1-\eta)}{1-\epsilon})(\|\triangle_A^i{\Sigma_\star^i}^{1/2}\|_F^2+\|\triangle_B^i{\Sigma_\star^i}^{1/2}\|_F^2)+\frac{\eta^2(2\epsilon+\epsilon^2)}{(1-\epsilon)^2}(\|\triangle_B^i{\Sigma_\star^i}^{1/2}\|_F^2+\|\triangle_A^i{\Sigma_\star^i}^{1/2}\|_F^2)\\
&\qquad + \eta^2(\frac{{\delta_{2r}^i}^2(2+\epsilon)^2}{2(1-\epsilon)^2}(2\|\triangle_A^i{\Sigma_\star^i}^{1/2}\|_F^2+2\|\triangle_B^i{\Sigma_\star^i}^{1/2}\|_F^2)+(P-1)\sum_{j\neq i}\frac{\textcolor{black}{\|C_i^TC_j\|_2^2}(2+\epsilon)^2}{2(1-\epsilon)^2}(2\|\triangle_A^j{\Sigma_\star^j}^{1/2}\|_F^2\\
&\qquad +2\|\triangle_B^j{\Sigma_\star^j}^{1/2}\|_F^2) +\sum_{j\neq i}\frac{2(1+\frac{\epsilon}{2})^2\delta_{2r}^i}{(1-\epsilon)^2}\textcolor{black}{\|C_i^TC_j\|}(2\|\triangle_A^j{\Sigma_\star^j}^{1/2}\|_F+2\|\triangle_B^j{\Sigma_\star^j}^{1/2}\|_F)(\|\triangle_A^i{\Sigma_\star^i}^{1/2}\|_F\\
&\qquad +\|\triangle_B^i{\Sigma_\star^i}^{1/2}\|_F))+2\eta(\frac{(1-\eta)\delta_{2r}^i(2+\epsilon)}{2(1-\epsilon)}(2\|\triangle_A^i{\Sigma_\star^i}^{1/2}\|_F^2+2\|\triangle_B^i{\Sigma_\star^i}^{1/2}\|_F^2)\\
&\qquad +\frac{\eta \delta_{2r}^i (2+\epsilon)}{2(1-\epsilon)^2}(2\|\triangle_B^i{\Sigma_\star^i}^{1/2}\|_F^2+2\|\triangle_A^i{\Sigma_\star^i}^{1/2}\|_F^2)+\frac{(2+\epsilon)(1-\eta)}{2(1-\epsilon)}\sum_{j\neq i}\|C_i^TC_j\|_2(\|\triangle_A^j {\Sigma_\star^j}^{1/2}\|_F^2\\
&\qquad +\|\triangle_B^j {\Sigma_\star^j}^{1/2}\|_F^2+\|\triangle_A^i{\Sigma_\star^i}^{1/2}\|_F^2+\|\triangle_B^i{\Sigma_\star^i}^{1/2}\|_F^2)+\frac{\eta(2+\epsilon)}{2(1-\epsilon)^2}\sum_{j\neq i} \|C_i^TC_j\|_2(\|\triangle_A^j {\Sigma_\star^j}^{1/2}\|_F^2+\\&\qquad \|\triangle_B^j {\Sigma_\star^j}^{1/2}\|_F^2+\|\triangle_B^i{A_\star^i}^T\|_F^2+\|\triangle_A^i{A_\star^i}^T\|_F^2))\\
&\leq ((1-\eta)^2+\frac{2\epsilon\eta(1-\eta)}{1-\epsilon}+\frac{\eta^2(2\epsilon+\epsilon^2)}{(1-\epsilon)^2}+\frac{\eta^2{\delta_{2r}^i}^2(2+\epsilon)^2}{(1-\epsilon)^2})\text{dist}^2(F_t^i,F_\star^i)\\
&\qquad+\sum_{j\neq i}\frac{(P-1)\eta^2(2+\epsilon)^2\|C_i^TC_j\|_2^2}{(1-\epsilon)^2}\text{dist}^2(F_t^j,F_\star^j)\\
&\qquad +\sum_{j\neq i}\frac{8\eta^2(1+\frac{\epsilon}{2})^2\delta_{2r}^i}{(1-\epsilon)^2}\textcolor{black}{\|C_i^TC_j\|}\text{dist}(F_t^i,F_\star^i)\text{dist}(F_t^j,F_\star^j)+\frac{2\eta(1-\eta)\delta_{2r}^i(2+\epsilon)}{(1-\epsilon)}\text{dist}^2(F_t^i,F_\star^i)\\
&\qquad +\frac{2\eta^2 \delta_{2r}^i (2+\epsilon)}{(1-\epsilon)^2}\text{dist}^2(F_t^i,F_\star^i) +\sum_{j\neq i}\frac{\eta(2+\epsilon)(1-\eta)}{1-\epsilon}\|C_i^TC_j\|_2\text{dist}^2(F_t^j,F_\star^j)\\
&\qquad +\frac{\eta(2+\epsilon)(1-\eta)\sum_{j\neq i}\textcolor{black}{\|C_i^TC_j\|_2}}{(1-\epsilon)}\text{dist}^2(F_t^i,F_\star^i) +\sum_{j\neq i}\frac{\eta^2(2+\epsilon)}{(1-\epsilon)^2}\|C_i^TC_j\|_2\text{dist}^2(F_t^j,F_\star^j)\\
&\qquad +\sum_{j\neq i}\frac{\eta^2(2+\epsilon)\|C_i^TC_j\|_2}{(1-\epsilon)^2}\text{dist}^2(F_t^i,F_\star^i)\\
&=((1-\eta)^2+\frac{2\epsilon\eta(1-\eta)}{1-\epsilon}+\frac{\eta^2(2\epsilon+\epsilon^2)}{(1-\epsilon)^2}+\frac{\eta^2{\delta_{2r}^i}^2(2+\epsilon)^2}{(1-\epsilon)^2}+\frac{2\eta(1-\eta)\delta_{2r}^i(2+\epsilon)}{(1-\epsilon)}+\frac{2\eta^2 \delta_{2r}^i (2+\epsilon)}{(1-\epsilon)^2}\\
&\qquad +\frac{\eta(2+\epsilon)(1-\eta)\sum_{j\neq i}\textcolor{black}{\|C_i^TC_j\|_2}}{(1-\epsilon)}+\frac{\eta^2(2+\epsilon)\sum_{j\neq i}\|C_i^TC_j\|_2}{(1-\epsilon)^2})\text{dist}^2(F_t^i,F_\star^i)\\
\end{aligned}
\end{equation*}
\begin{equation*}
\begin{aligned}
&\qquad +\sum_{j\neq i}(\frac{(P-1)\eta^2\textcolor{black}{\|C_i^TC_j\|^2}(2+\epsilon)^2}{(1-\epsilon)^2}+\frac{\eta(2+\epsilon)(1-\eta)\|C_i^TC_j\|_2}{(1-\epsilon)}+\frac{\eta^2(2+\epsilon)\|C_i^TC_j\|_2}{(1-\epsilon)^2})\text{dist}^2(F_t^j,F_\star^j)\\
&\qquad +\sum_{j\neq i}\frac{8\eta^2(1+\frac{\epsilon}{2})^2\delta_{2r}^i}{(1-\epsilon)^2}\|C_i^TC_j\|_2\text{dist}(F_t^i,F_\star^i)\text{dist}(F_t^j,F_\star^j)\\
& \text{Let $\sigma_{ij}$ denote $\|C_i^TC_j\|_2$,}
\\
&\leq ((1-0.6\eta)^2+\frac{7}{3}\textcolor{black}{\sum_{j\neq i}\sigma_{ij}}\eta+\frac{7}{27}\textcolor{black}{\sum_{j\neq i}\sigma_{ij}}\eta^2)\text{dist}^2(F_t^i,F_\star^i)+\sum_{j\neq i}(\frac{7}{3}\textcolor{black}{\sigma_{ij}}\eta+(\frac{7}{27}\textcolor{black}{\sigma_{ij}}+\frac{49}{9}(P-1)\textcolor{black}{\sigma_{ij}^2})\eta^2)\text{dist}^2(F_t^j,F_\star^j)\\&\qquad +\sum_{j\neq i}0.3\textcolor{black}{\sigma_{ij}}\eta^2\text{dist}(F_t^i,F_\star^i)\text{dist}(F_t^j,F_\star^j)\\
& \text{Let $\sigma=\max(\sigma_{ij})$ over all $j,$ when $\sigma,\eta\leq 1,$}\\
&\leq ((1-0.6\eta)^2+3(P-1)\textcolor{black}{\sigma}\eta)\text{dist}^2(F_t^i,F_\star^i)+\sum_{j\neq i}(3\textcolor{black}{\sigma}+6(P-1)\textcolor{black}{\sigma^2})\eta \text{dist}^2(F_t^j,F_\star^j)\\&\qquad +\sum_{j\neq i}0.3\textcolor{black}{\sigma}\eta\text{dist}(F_t^i,F_\star^i)\text{dist}(F_t^j,F_\star^j)
\end{aligned}
\end{equation*}
\newline
WLOG assume $\text{dist}^2(F_t^k,F_\star^k)\geq \text{dist}^2(F_t^j,F_\star^j)$
for any $j\neq k.$ Then when $\sigma\leq \min(1,\frac{0.12}{7P(P+1)})$,
\begin{equation*}
\begin{aligned}
&((1-0.6\eta)^2+3(P-1)\textcolor{black}{\sigma}\eta)\text{dist}^2(F_t^k,F_\star^k)+\sum_{j\neq i}(3\textcolor{black}{\sigma}+6(P-1)\textcolor{black}{\sigma^2})\eta \text{dist}^2(F_t^k,F_\star^k)+\sum_{j\neq i}0.3\textcolor{black}{\sigma}\eta\text{dist}^2(F_t^k,F_\star^k)\\&\leq (1-0.5\eta)^2\text{dist}^2(F_t^k,F_\star^k).
\end{aligned}
\end{equation*}
Therefore,
\[\max_i(\text{dist}^2(F_{t+1}^i,F_\star^i))\leq (1-0.5\eta)^2\max_i(\text{dist}^2(F_t^i,F_\star^i)).\]
\end{proof}
The proof of Theorem \ref{main_thm} is a simple combination of Lemma \ref{lemma1} and Lemma \ref{lemma2}.  We build on \cite{tong2021accelerating} for our proof. Note in \cite{jia2023preconditioning}, the authors provide a proof for global convergence for scaled GD method for least squares matrix decomposition problem. Since the proof detail there is closely tailed to the objective, we stick to the local convergence proof in \cite{tong2021accelerating} which more resembles the problem we are interested in. 
\section{SGD with Gradient Scaling}\label{sgd_sec}
\begin{algorithm*}[ht!]
\caption{Pseudocode of scaled GD in  PyTorch.}
\label{alg:scaled_gd}
\algcomment{\fontsize{7.2pt}{0em}\selectfont \texttt{pairwise}: read every two elements in a list 
}
\definecolor{codeblue}{rgb}{0.25,0.5,0.5}
\definecolor{codered}{rgb}{0.8,0.2,0.2}
\lstset{
  backgroundcolor=\color{white},
  basicstyle=\fontsize{7.2pt}{7.2pt}\ttfamily\selectfont,
  columns=fullflexible,
  breaklines=true,
  captionpos=b,
  commentstyle=\fontsize{7.2pt}{7.2pt}\color{codeblue},
  keywordstyle=\fontsize{7.2pt}{7.2pt},
  emph={{,},]}, 
emphstyle=\color{codered}, 
escapeinside={<@}{@>}
}
\begin{lstlisting}[language=python]
# group trainable parameters into LoRA pairs in train.py
<@\textcolor{codered}{for LoRA$\_$A, LoRA$\_$B in pairwise(trainable$\_$parameter):}@>
     <@\textcolor{codered}{param$\_$groups.append($\{$"params": [LoRA$\_$A,LoRA$\_$B], "lr": learning$\_$rate$\}$)}@>
     
# apply preconditioner in optimizer.py
for group in param_groups:
      A, B = group["params"]  
      dA, dB = group["params"].grad
      # precondition gradients
      <@\textcolor{codered}{dA$\_$scaled =inverse(B.T@B+delta*torch.eye(r)).mm(dA) }@> 
      <@\textcolor{codered}{dB$\_$scaled =dB.mm(inverse(A@A.T+delta*torch.eye(r))) }@> 
      # update parameters
      A.add_(dA_scaled, -group['lr']) 
      B.add_(dB_scaled, -group['lr'])
\end{lstlisting}
\end{algorithm*}

\section{More on Runtime Comparison}\label{runtime_sec}
Figure \ref{runtime_256} shows runtime comparison for fine-tuning GPT-2 model with LoRA trained with differernt optimizers, $r=256$ is adopted. Note though the runtime gap between scaled optimizers and unscaled ones increases compared to $r=4$ shown in Section \ref{runtime_section}, the increment is still marginal.
 \begin{figure*}[ht!]
 \centering
\includegraphics[width=0.7\linewidth]{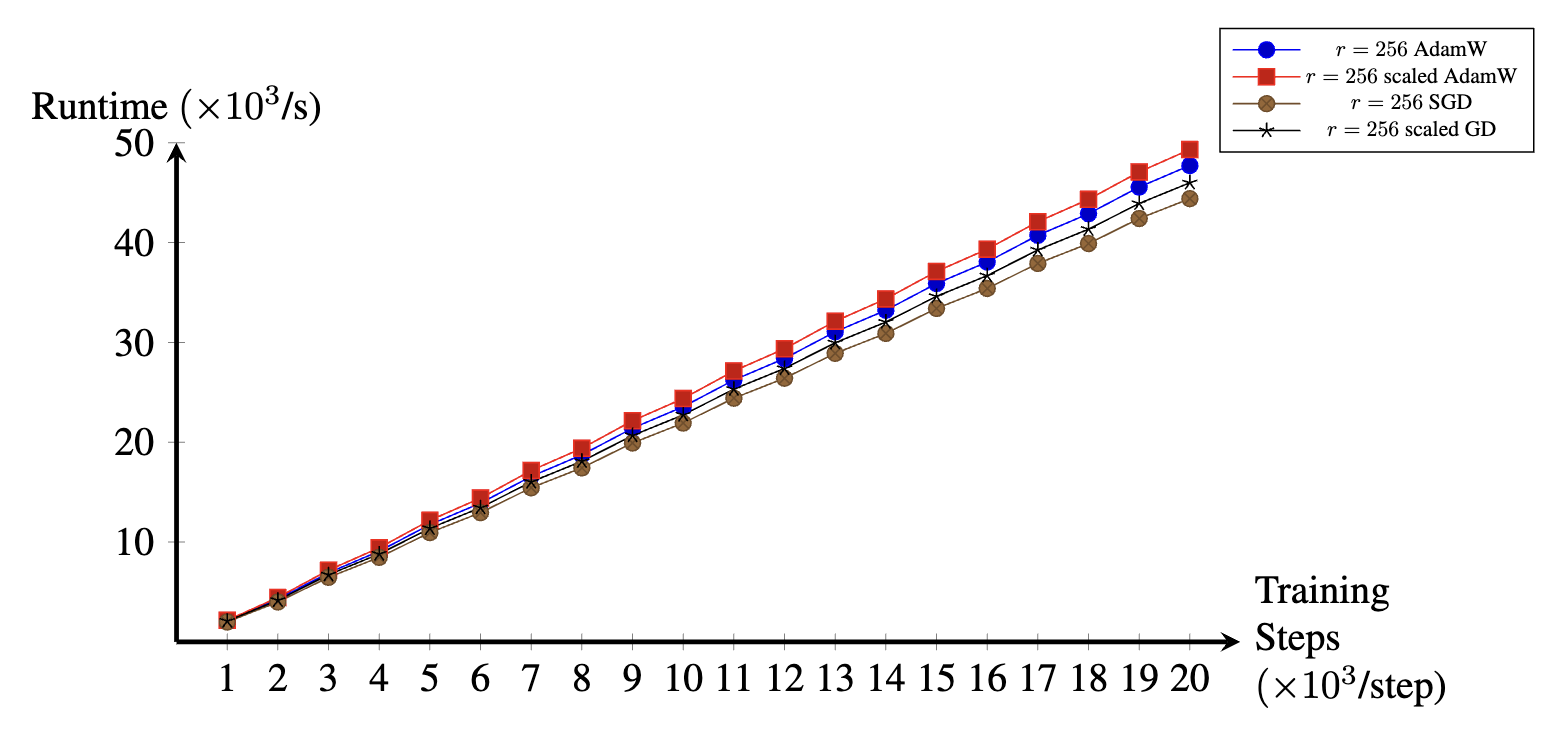}
\caption{Runtime for LoRA fine-tuning GPT-2 medium model with rank $r=256$ with different optimizers. Our scaled methods introduce marginal runtime overhead and train as fast as unscaled methods.  See Section \ref{gpt2_section} for experimental details.}\label{runtime_256}
\end{figure*}
\section{Supplementary Experiments for Language Models}\label{simulation_supp}
\subsection{GPT-2}\label{gpt2_supp}

\subsubsection{Experimental Results for Different LoRA Ranks}\label{rank_supp}
For experiments for varying LoRA ranks, we follow exact the same setting as in original LoRA project \cite{hu2021lora}. We experiment with medium-size GPT-2 \cite{Radford2019LanguageMA} model with hyperparameters listed in Table \ref{gpt2_hyper}, where the learning rates for different methods are individually tuned by grid search except for AdamW, which we follow default setting in LoRA report \cite{hu2021lora}.  We train with a linear learning
rate schedule for 5 epochs. We note that we also tune  hyperparameters $\beta_1,\beta_2$ for AdamW-type methods and find that lower $\beta_1,\beta_2$ values are beneficial to scaled AdamW method. See Table \ref{table3} for experimental results. We test with LoRA rank $r=1,4,8$ and method with scaled gradient always performs better than its non-scaled gradient counterpart for all ranks and on most evaluation metrics. We set regularization factor to be $\sigma=1e-6$.



\begin{table*}[ht!]
\centering
{
\begin{tabular}{c|cccccccccccc}
\\ \thickhline
Method   & \multicolumn{3}{c}{SGD} & \multicolumn{3}{c}{scaled GD} & \multicolumn{3}{c}{AdamW} & \multicolumn{3}{c}{scaled AdamW}\\
Rank & $1$ & $4$ & $8$\quad \vline& $1$ & $4$ & $8$\quad\vline & $1$ & $4$ & $8$ \quad\vline & $1$ & $4$ & $8$ \quad  \\
\midrule 
&\multicolumn{12}{c} {Training}\\
 \hline
 Weight decay & \multicolumn{12}{c} {0.01}\\
 Dropout Prob & \multicolumn{12}{c} {0.1}\\
 Batch Size & \multicolumn{12}{c} {8}\\
 $\#$ Epoch & \multicolumn{12}{c} {5}\\
 Warmup Steps & \multicolumn{12}{c} {500}\\
 LR Scheduler & \multicolumn{12}{c} {Linear}\\
 Label Smooth & \multicolumn{12}{c} {0.1}
\\
LR (tuned, $\times 10^{-3}$)& $60$ & \multicolumn{2}{c}{$90$} & \multicolumn{2}{c}{$20$} &  $30$   & \multicolumn{3}{c}{$0.2$} & $0.5$ & $0.8$ & $2$\\
AdamW $\beta_1$ & \multicolumn{6}{c}{$/$} & \multicolumn{3}{c}{0.9} & \multicolumn{3}{c}{0.7}\\
AdamW $\beta_2$ & \multicolumn{6}{c}{$/$} & \multicolumn{3}{c}{0.999} & \multicolumn{3}{c}{0.8}\\
LoRA $\alpha$ & \multicolumn{12}{c} {32}\\
 \midrule
 & \multicolumn{12}{c}{Inference}\\
 \hline
 Beam Size & \multicolumn{12}{c}{10}\\
 Length Penalty & \multicolumn{12}{c}{0.8}\\
 No Repeat Ngram Size & \multicolumn{12}{c}{4} \\
 \thickhline

\end{tabular}
}
\caption{Hyperparameters for GPT-2 model fine-tuning }\label{gpt2_hyper}
\end{table*}

\begin{table*}
\centering
{
\begin{tabular}{l|c|ccccc}
 \thickhline
\multirow{2}{*}{Method} & 
\multirow{2}{*}{Rank} & 
  \multicolumn{5}{c}{E2E}  \\ 
 &  &  BLEU & NIST &MET & ROUGE-L & CIDEr\\
 \hline
SGD  & 1 & 35.9 & 5.09 & 25.3 & 46.2  & 0.48 \\ 
scaled GD (ours) & 1 & 68.2 & 8.65 & 45.7 & 69.6 & 2.44\\ 
  AdamW & 1 & 69.9 & 8.80 & \textbf{46.5} & 71.4 & 2.48\\ 
  scaled AdamW (ours) & 1 & \textbf{70.1} & \textbf{8.82} &\textbf{46.5} & \textbf{71.7} & \textbf{2.51}\\ 
  \hline
SGD  & 4 &  66.6 & 8.54 & 44.2 & 68.2 & 2.32\\ 
scaled GD (ours) & 4 & 69.2 & 8.71 & 46.3 & 70.9 & 2.48 \\ 
  AdamW & 4 & 68.9 & 8.69 & 46.5 & 71.3 & 2.51\\ 
  scaled AdamW (ours) & 4 &  \textbf{69.6} & \textbf{8.77} & \textbf{46.6} & \textbf{71.8} & \textbf{2.52}\\ 
  \hline
  SGD  & 8 & 65.8 & 8.46 & 43.5 & 68.7 & 2.33 \\ 
scaled GD (ours) & 8 & 69.6 & 8.78 & 46.4 & 70.8 & 2.48 \\ 
  AdamW & 8 & 69.6 & 8.74 & \textbf{46.7} & \textbf{71.8} & \textbf{2.53}\\ 
  scaled AdamW (ours) & 8 & \textbf{70.1} & \textbf{8.82} & 46.6 & \textbf{71.8} & \textbf{2.53} \\ 
  \thickhline
\end{tabular}
}
\caption{Experiments for GPT-2 medium model on E2E NLG challenge with different LoRA ranks. Our scaled optimizers outperform unscaled optimizers for all LoRA ranks being tested and on most evaluation metrics. Moreover, scaled GD method behaves close to AdamW method. See Appendix \ref{rank_supp} for experimental details.}\label{table3}
\end{table*}

\subsubsection{Experimental Results for Different Model Sizes}\label{size_supp}
For GPT-2 model, we also experiment with different model sizes. LoRA rank $r$ is fixed to be $4$. We use the same training hyperparameters as listed in Table  \ref{gpt2_hyper}. See Table \ref{model_size} for final scores. Note our scaled gradient methods always outperform their unscaled gradient counterparts for different model sizes and on most evaluation metrics, which shows the superiority of the introduced preconditioner. Furthermore, there is usually significant performance gap between SGD method and AdamW method while our scaled GD method is able to obtain scores comparable to AdamW without requiring momentum terms. Our scaled GD method indeed closes the gap between SGD and AdamW. 


\begin{table*}[ht!]
\centering
{
\noindent\makebox[\textwidth]{\begin{tabularx}{0.93\textwidth}{l|c|c|ccccc}
\thickhline
\multirow{2}{*}{Method} & 
\multirow{2}{*}{Model} & 
\multirow{2}{*}{$\#$ Trainable Parameters} &
  \multicolumn{5}{c}{E2E}  \\ 
& &  & BLEU & NIST &MET & ROUGE-L & CIDEr\\
  \hline
SGD  & GPT-2 S & 0.15M & 54.8 & 4.56 & 34.0 & 63.3 & 1.29  \\ 
scaled GD (ours) & GPT-2 S & 0.15M & 68.5 & 8.72 & 45.5 & 69.4 & 2.40 \\ 
  AdamW & GPT-2 S & 0.15M &  69.1 & 8.75 & 46.0 & 70.5 & 2.47\\ 
  scaled AdamW (ours) & GPT-2 S & 0.15M & \textbf{69.5} & \textbf{8.80} & \textbf{46.2} & \textbf{70.9} & \textbf{2.48}\\ 
  \hline
SGD  & GPT-2 M & 0.39M & 66.6 & 8.54 & 44.2 & 68.2 & 2.32 \\ 
scaled GD (ours) & GPT-2 M & 0.39M & 69.2 & 8.71 & 46.3 & 70.9 & 2.48 \\ 
  AdamW & GPT-2 M & 0.39M & 68.9 & 8.69 & 46.5 & 71.3 & 2.51\\ 
  scaled AdamW (ours) & GPT-2 M & 0.39M & \textbf{69.6} & \textbf{8.77} & \textbf{46.6} & \textbf{71.8} & \textbf{2.52}\\ 
\thickhline
\end{tabularx}}
}
\caption{Experiments for GPT-2 models with different sizes with LoRA rank $r=4$ on E2E NLG challenge. Our scaled optimizers behave better than unscaled ones and scaled GD obtains scores comparable to AdamW. See Appendix \ref{size_supp} for experimental details. }\label{model_size}
\end{table*}
\subsubsection{Experimental Results for  Different Datasets}\label{dataset_supp}
We experiment with also WebNLG \cite{gardent-etal-2017-webnlg} and DART \cite{nan-etal-2021-dart} datasets with GPT-2 medium-size model with LoRA rank $r=4$. We use exactly the same training hyperparameters as listed in Table \ref{gpt2_hyper}. WebNLG is a popular dataset for data-to-text evaluation introduced by \cite{gardent-etal-2017-webnlg} which includes 22K examples from 14 distinct categories. Among these categories, five categories are presented only at test time and thus the evaluation is divided into ``seen" (S), ``unseen" (U), ``all" (A) three types depending on whether the five categories are included in test time or not. DART is another data-to-text dataset introduced by \cite{nan-etal-2021-dart} which involves 82K examples. The evaluation metrics being used are BLEU, MET, and TER with higher scores being better for first two metrics and the lower the better for TER. The experimental result is presented in Table \ref{dataset_exp} and we observe that scaled GD significantly improves SGD performance for both datasets on all evaluation metrics, and so does scaled AdamW for AdamW.

\begin{table*}[ht!]
\centering
{
\noindent\makebox[\textwidth]{\begin{tabularx}{0.93\textwidth}{c|ccc|ccccccccc}
\thickhline
\multirow{3}{*}{Method} & \multicolumn{3}{c|}{DART} & \multicolumn{9}{c}{WebNLG}\\
& \multirow{2}{*}{BLEU$\uparrow$} & \multirow{2}{*}{MET$\uparrow$} & \multirow{2}{*}{TER$\downarrow$} & \multicolumn{3}{c}{BLEU$\uparrow$} & \multicolumn{3}{c}{MET$\uparrow$} & \multicolumn{3}{c}{TER$\downarrow$}\\
&&&& U & S & A & U & S & A & U & S & A  
\\ 
  \hline
SGD  & 43.2  & .36  & .50  & 45.5 & 58.1 & 52.4 & .36 & .42 & .39 &  \textbf{.45} & .36 & .40 \\ 
scaled GD (ours) & \textcolor{black}{46.1} & \textcolor{black}{.38} & \textcolor{black}{.48} & 46.3 & 61.7  & 54.8 & .37 & .44 &  .41 & \textbf{.45} & .34 &  .39 \\ 
  AdamW & 47.1  &  .38 & \textbf{.47}  & 45.0 & 64.1 &  55.5 &\textbf{.38}  & \textbf{.45} &  \textbf{.42} & .47 & \textbf{.32} &  .39 \\ 
  scaled AdamW (ours) & \textbf{47.9} & \textbf{.39} & \textbf{.47}  &\textbf{46.8} & \textbf{64.2} & \textbf{56.3} &\textbf{.38} & \textbf{.45} &  \textbf{.42} & .46 & \textbf{.32} &  \textbf{.38}  \\ 
   \thickhline
\end{tabularx}}
}
\caption{Experiments for GPT-2 medium model on NLG challenge with DART dataset and WebNLG dataset. Our scaled optimizers improve unscaled optimizers uniformly on all evaluation metrics for both datasets. See Appendix \ref{dataset_supp} for experiment, datasets, and metric details. }\label{dataset_exp}
\end{table*}

\subsection{Mistral 7B}\label{mistral_append_sec}

Mistral 7B is a pretty new model up-to-date and there is no well-established code base we can follow. For quicker training, we use $4$-bit quantized version of Mistral 7B V0.1 as our base model and LoRA factors are injected to each linear layer with rank $r=16.$ We train for $5$ total epochs with batch size $8$. For fine-tuning 4-bit quantized Mistral 7B V0.1 model for GLUE benchmark, we exploit HuggingFace transformers trainer class. All training arguments are default there except for the batch size, training epoch, and optimizer-related settings which are customized for our experiments. See Table \ref{misrtal_train} for training hyperparameter choices.  The $\beta$'s and $\epsilon$ for AdamW-type methods are defaultly used in original LoRA project \cite{hu2021lora}. For learning rate choices,  we view that values ranging from $2e-5$ to $2e-4$ have been empirically used for fine-tuning Mistral 7B for other tasks. We follow \cite{Mistral_Fine_Tuning} and set $lr=5e-5$ for AdamW-type methods. For larger datasets including mnli, \textcolor{black}{qnli}, and \textcolor{black}{qqp}, $lr=5e-5$ results in  NaN loss thus we tune learning rate individually with grid search for each method. Since SGD has never been used for training Mistral 7B, we empirically find $lr=5e-3$ to be a reasonable learning rate. For larger datasets including mnli, \textcolor{black}{qnli}, and qqp, we still tune the learning rate with grid search. Learning rate scheduler, warmup steps, warmup ratios, and max grad norm are all default in HuggingFace trainer class. Weight decay value $0.01$ is what has been used in original LoRA project for GPT-2 fine-tuning.   We use the same LoRA-related parameters and model quantization configuration as in \cite{Mistral_Fine_Tuning}. 


\begin{table*}[ht!]
\centering
{
\begin{tabular}{c|cccc}
 \thickhline
Method & SGD & scaled GD & AdamW & scaled AdamW \\
 \hline
 
train batch size & \multicolumn{4}{c}{$8$} \\
seed (default) & \multicolumn{4}{c}{$42$}\\
AdamW $(\beta_1,\beta_2)$ & \multicolumn{2}{c}{$/$} & \multicolumn{2}{c}{$(0.9,0.999)$} \\
AdamW $\epsilon$ & \multicolumn{2}{c}{$/$} & \multicolumn{2}{c}{$1e-6$} \\
lr & \multicolumn{2}{c}{$5e-3$} & \multicolumn{2}{c}{$5e-5$} \\
lr (mnli) & $5e-3$ & $1e-3$ & $5e-6$ & $3e-5$ \\
lr (qqp) & $5e-3$ & $4e-3$ &  $5e-6$ & $3e-5$ \\
\textcolor{black}{lr (qnli)} & $5e-3$ & $9e-4$& \multicolumn{2}{c}{$1e-5$} \\
lr scheduler & \multicolumn{4}{c}{linear} \\
num epoch  & \multicolumn{4}{c}{$5$}  \\
warmup steps $\&$ warmup ratios & \multicolumn{4}{c}{$0$} \\
weight decay & \multicolumn{4}{c}{$0.01$} \\
max grad norm & \multicolumn{4}{c}{$1$} \\
LoRA rank & \multicolumn{4}{c}{$16$} \\
LoRA $\alpha$ & \multicolumn{4}{c}{$16$} \\
LoRA dropout  & \multicolumn{4}{c}{$0.05$} \\
Load in 4-bit & \multicolumn{4}{c}{True}\\
4bit quantization type & \multicolumn{4}{c}{nf4} \\
4bit dtype & \multicolumn{4}{c}{bfloat$16$}\\
  \thickhline
\end{tabular}
}
\caption{Hyperparameters for Mistral 7B model fine-tuning }\label{misrtal_train}
\end{table*}
\section{Supplementary Experiments for Diffusion Models}\label{diffusion_supp}
\subsection{Stable Diffusion}\label{stable_diffusion_append}
For our object generation experiment, we follow the popular custom diffusion repository \cite{Ryu2023}. We follow all default settings for both training and inference steps except for the optimizer component. We use ``a photo of $\langle V_{\text{object}} \rangle$" as all object image captions. In the original repository, AdamW is used as default optimizer with learning rate $5e-5$ for text-encoder tuning and $1e-4$ for U-Net tuning. The pretrained model being used is Stable Diffusion V1.5 \cite{Rombach_2022_CVPR}. For training procedure, we use constant learning rate scheduler with zero learning rate warmup steps, which is the default setting. Max training step is set to $4000$. For sampling procedure, we set number of inference steps to be $50$ and guidance scale to be $7$, which is used as default in the experiment notebook provided. Figure \ref{table_main} shows generation results for a yellow chair with training images containing the target chair in color blue. 
 AdamW is able to generate the target yellow chair only for learning rate $1e-6$ while our method generates desired images for all learning rates being tested.   Figure  \ref{dog_append} shows generation results for  dog object. Our method generates images better capturing the prompt, i.e, a dog wearing a hat. For large learning rate $1e-2$, AdamW only generates black images while no black images are observed for scaled AdamW generation, which again verifies the robustness of our  scaled optimizers.

\begin{figure*}[h]
 \centering
\includegraphics[width=1.0\linewidth]{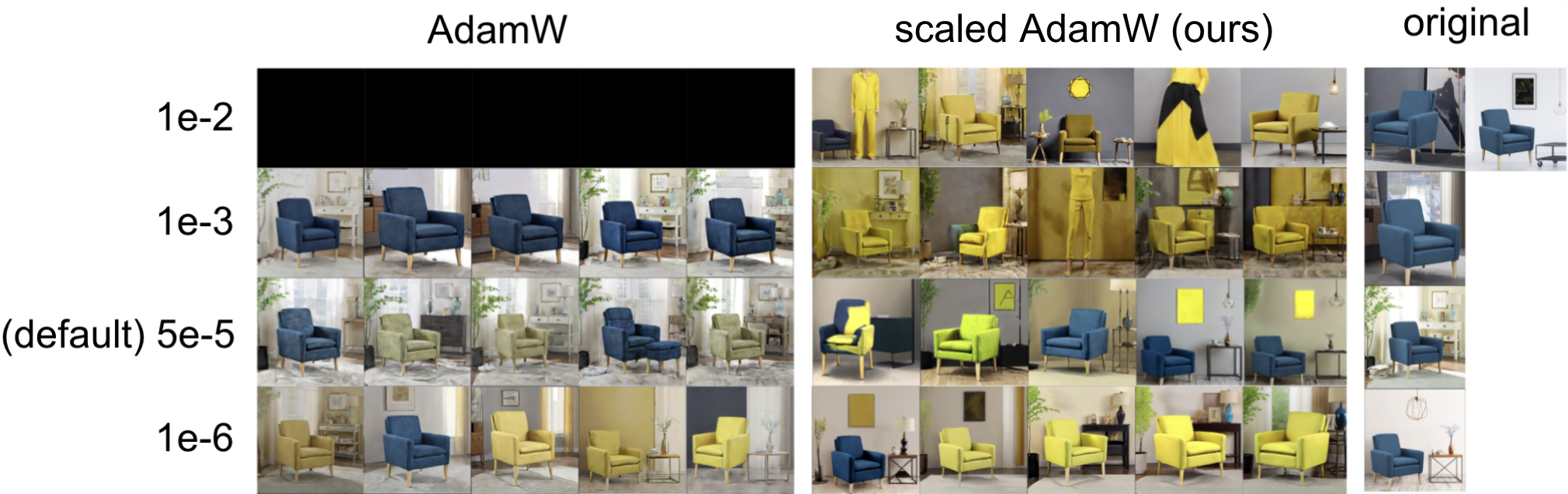}
\caption{
Generation results for prompt ``a yellow $\langle V_{\text{chair}} \rangle$'' after fine-tuning on $5$ blue chair images  of the Stable Diffusion V1.5 model. We vary text-encoder learning rates with U-Net learning rate fixed to default value $1e-4.$ No black images are observed for our method's generation and AdamW generates only black images for large learning rates. Our method (scaled AdamW) generates photos better capturing the prompt and is more robust to learning rate changes. See Appendix \ref{stable_diffusion_append} for experimental details.
}\label{table_main}
\end{figure*}

\begin{figure*}[h]
 \centering
\includegraphics[width=0.8\linewidth]{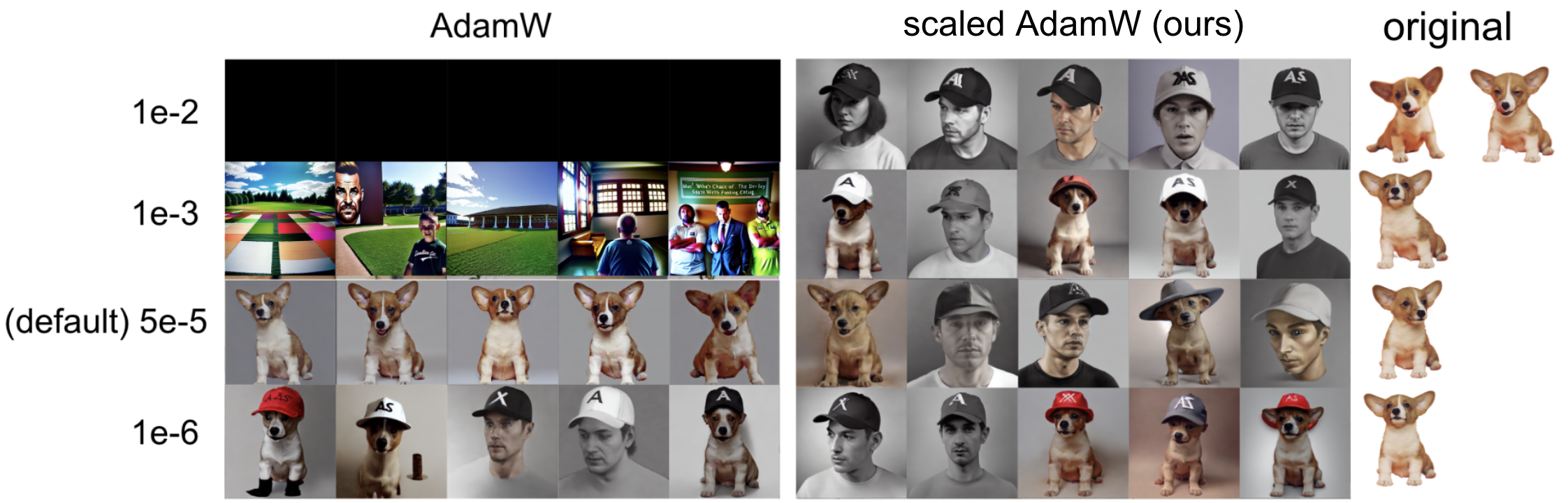}
\caption{Generation results for prompt ``$\langle V_{\text{dog}} \rangle$ wearing a hat'' after fine-tuning on $5$ dog images of the Stable Diffusion V1.5 model. See Appendix \ref{stable_diffusion_append} for experimental details. Our method (scaled AdamW) generates images better capturing the prompt, i.e., a dog wearing a hat, and is 
 more robust to learning rate changes.}\label{dog_append}
 \vspace{-0.5cm}
\end{figure*}

\subsection{Mix-of-Show}\label{mixshow_append}
For our experiment for face generation tasks, we base our experiments on Mix-of-Show \cite{gu2023mixofshow} repository which we find able to generate high-quality face images and is better for visualization comparison between different optimization methods we consider. We follow default settings for training and inference in \cite{gu2023mixofshow} with the exception that we turn off embedding tuning and only tune the text encoder and U-Net fraction where LoRA parameters are injected. The reason is that we find with embedding tuning, the effect of LoRA parameter is restricted and thus does no good to our comparison. In Mix-of-Show, Chilloutmix\footnote{\url{https://civitai.com/models/6424/chilloutmix}} is used as  pretrained model. LoRA rank is  set to $4$. DMP-Solver \cite{lu2022dpmsolver} is employed for sampling. See \cite{gu2023mixofshow} for more discussion on experimental details. Here we first fix step size $5e-4$ for both text encoder tuning and U-Net tuning and compare AdamW versus scaled AdamW with different LoRA parameter fusion coefficients. The default step size value used for Mix-of-Show is $1e-5$ and $1e-4$ for text encoder tuning and U-Net tuning respectively and default optimizer is AdamW. See Section \ref{gen_quality} for experimental results for different LoRA parameter fusion coefficients. Our scaled AdamW optimizer generates visually better images compared to AdamW optimizer. We then test with varying step size settings and demonstrate that our scaled gradient method is more robust to step size changes in Section \ref{lr_compare}.
\subsubsection{AdamW vs. scaled AdamW with Varying LoRA Parameter Fusion Coefficients}\label{gen_quality}
We experiment with Potter character and Hermione character. For Potter character,  we use 14 Potter images for training LoRA parameters with the character name replace by special token $\langle V_{\text{potter}} \rangle$ as what has been introduced in textual inversion \cite{gal2022image}, then in the sampling procedure, we use prompt with special character $\langle V_{\text{potter}} \rangle$ for generating images involving Potter character. Figure \ref{quality_1}, \ref{quality_2} and \ref{quality_3} show the generation results for three different prompts. The above two rows are for AdamW optimizer with the first row having LoRA parameter fusion coefficient $0.7$ and second row $1$ and the third and fourth rows correspond to scaled AdamW generation. LoRA parameter fusion coefficient represents $\alpha$ in $W=W_0+\alpha AB^T$ when merging LoRA weights. We observe that our scaled AdamW method is able to generate higher quality images compared to unscaled version for both $\alpha=0.7$ and $\alpha=1$. For Hermione character, we train with $15$ photos of Hermione following procedure described before. Figure \ref{hermoine_1} and \ref{hermoine_2} show the generation results. Still, our scaled AdamW optimizer produces higher quality images compared to AdamW optimizer.
 \begin{figure*}[h]
 \centering
\includegraphics[width=0.7\linewidth]{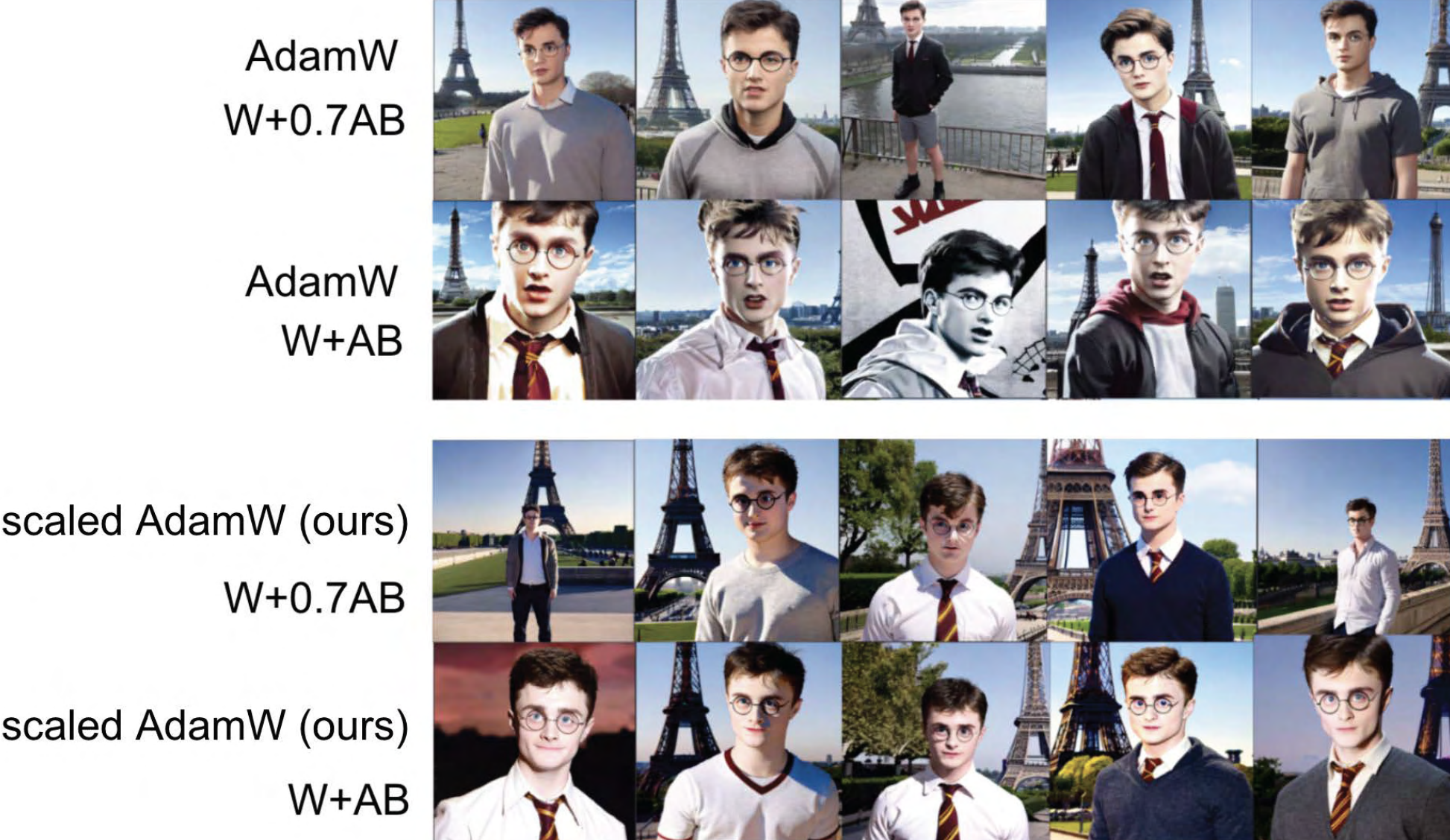}
\caption{Generation results for prompt ``a $\langle V_{\text{potter}} \rangle$ in front of eiffel tower'' after fine-tuning on $14$ Potter images. The above two rows are from AdamW optimizer and the bottom two rows are from our  scaled AdamW optimizer. The first and third rows correspond to LoRA parameter fusion coefficient $0.7$ and the second and fourth rows correspond to LoRA parameter fusion coefficient $1.0$. Our scaled AdamW method generates images of higher quality compared to AdamW. See Appendix \ref{gen_quality} for experimental details.}\label{quality_1}
\vspace{-0.2cm}
\end{figure*}
\begin{figure*}[ht!]
 \centering
\includegraphics[width=0.7\linewidth]{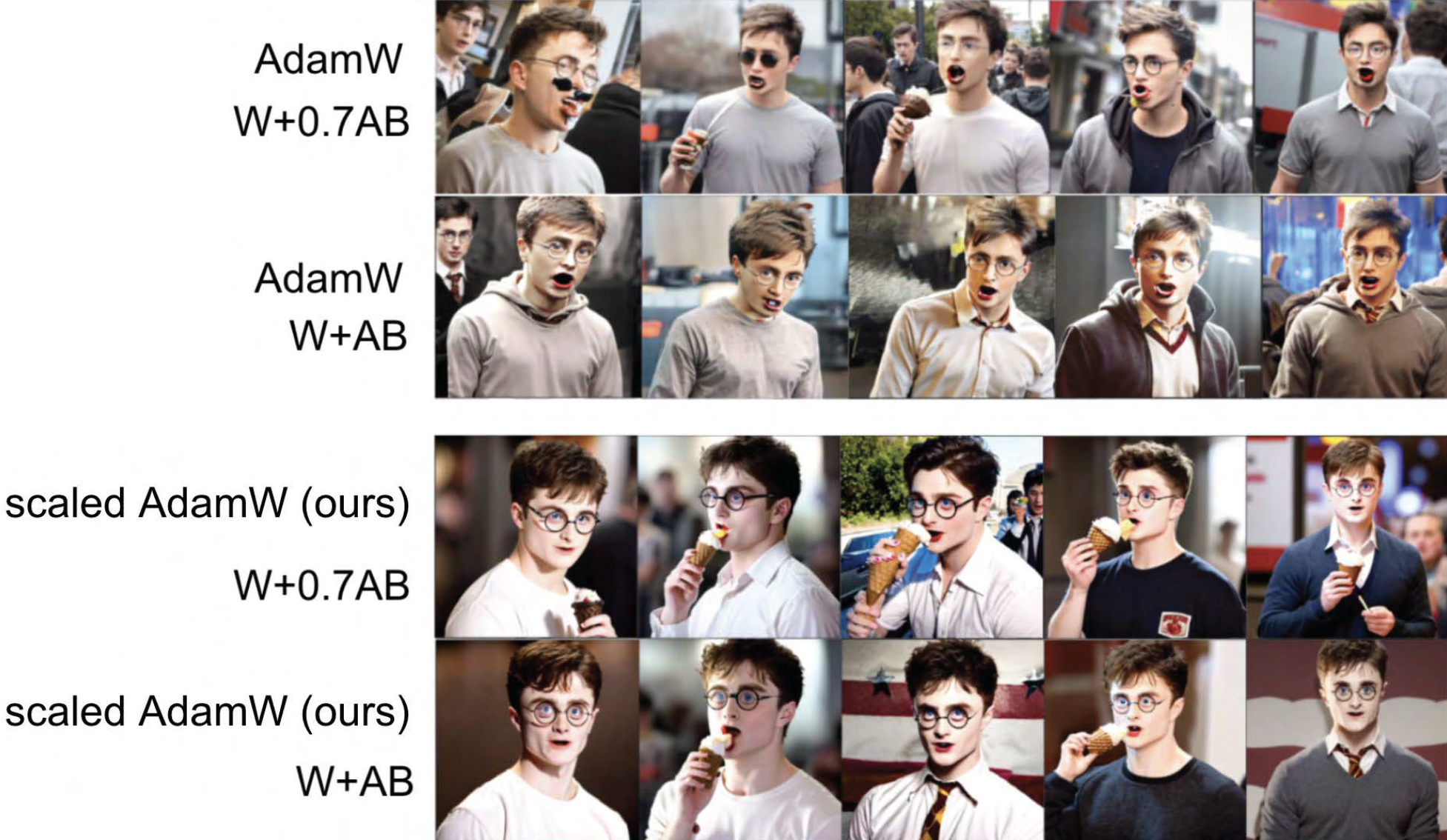}
\caption{Generation results for prompt ``$\langle V_{\text{potter}} \rangle $ eating an icecream '' after fine-tuning on $14$ Potter images. The above two rows are from AdamW optimizer and the bottom two rows are from our scaled AdamW optimizer. The first and third rows correspond to LoRA parameter fusion coefficient $0.7$ and the second and fourth rows correspond to LoRA parameter fusion coefficient $1.0$. Our scaled AdamW method generates images of higher quality compared to AdamW. See Appendix \ref{gen_quality} for experimental details.}\label{quality_2}
\end{figure*}
\begin{figure*}[ht!]
 \centering
\includegraphics[width=0.7\linewidth]{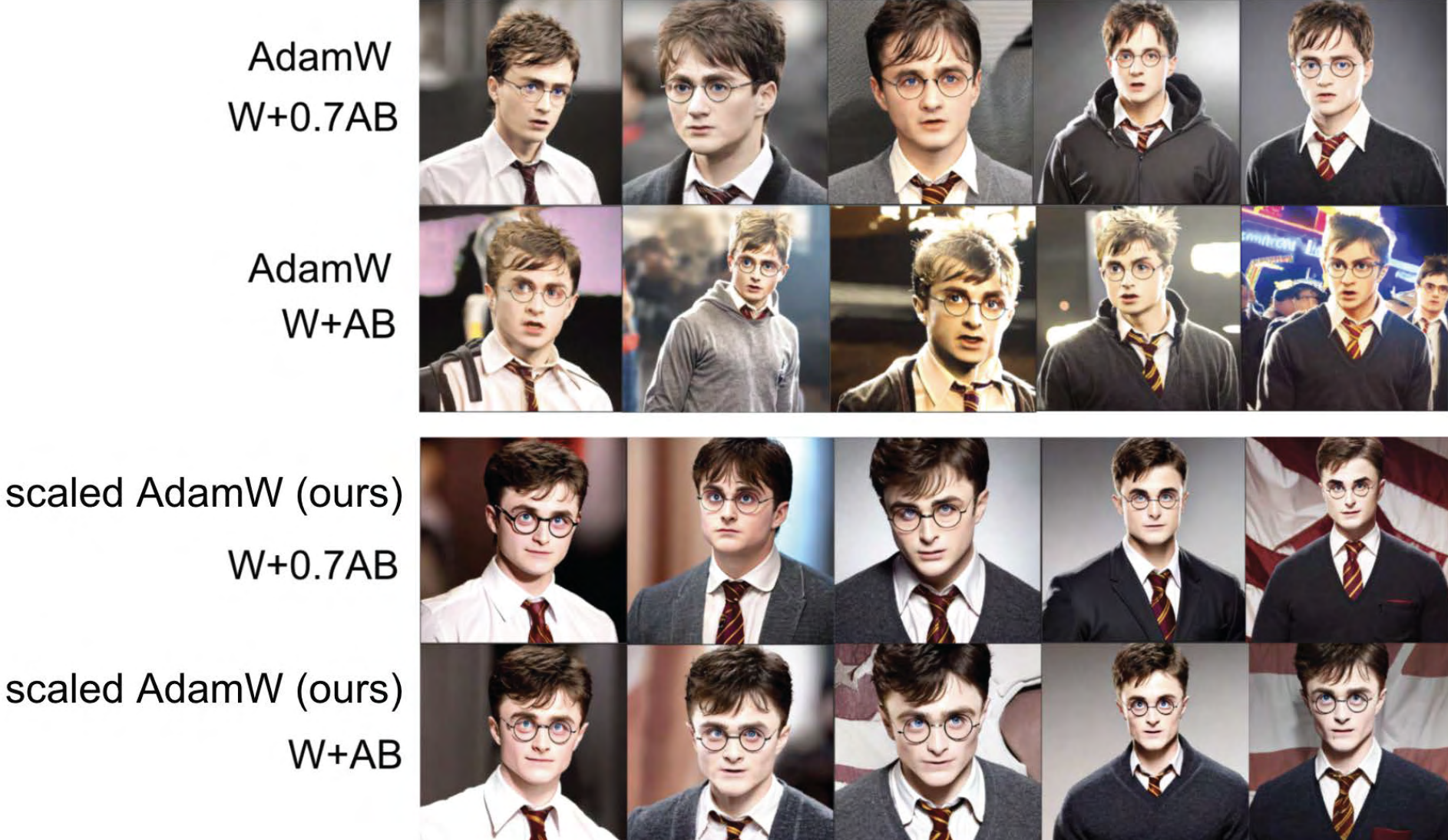}
\caption{ Generation results for prompt ``$\langle V_{\text{potter}} \rangle$'' after fine-tuning on $14$ Potter images. The above two rows are from AdamW optimizer and the bottom two rows are from our  scaled AdamW optimizer. The first and third rows correspond to LoRA parameter fusion coefficient $0.7$ and the second and fourth rows correspond to LoRA parameter fusion coefficient $1.0$. Our scaled AdamW method generates images of higher quality compared to AdamW. See Appendix \ref{gen_quality} for experimental details.}\label{quality_3}
\end{figure*}

\begin{figure*}[h]
 \centering
\includegraphics[width=0.7\linewidth]{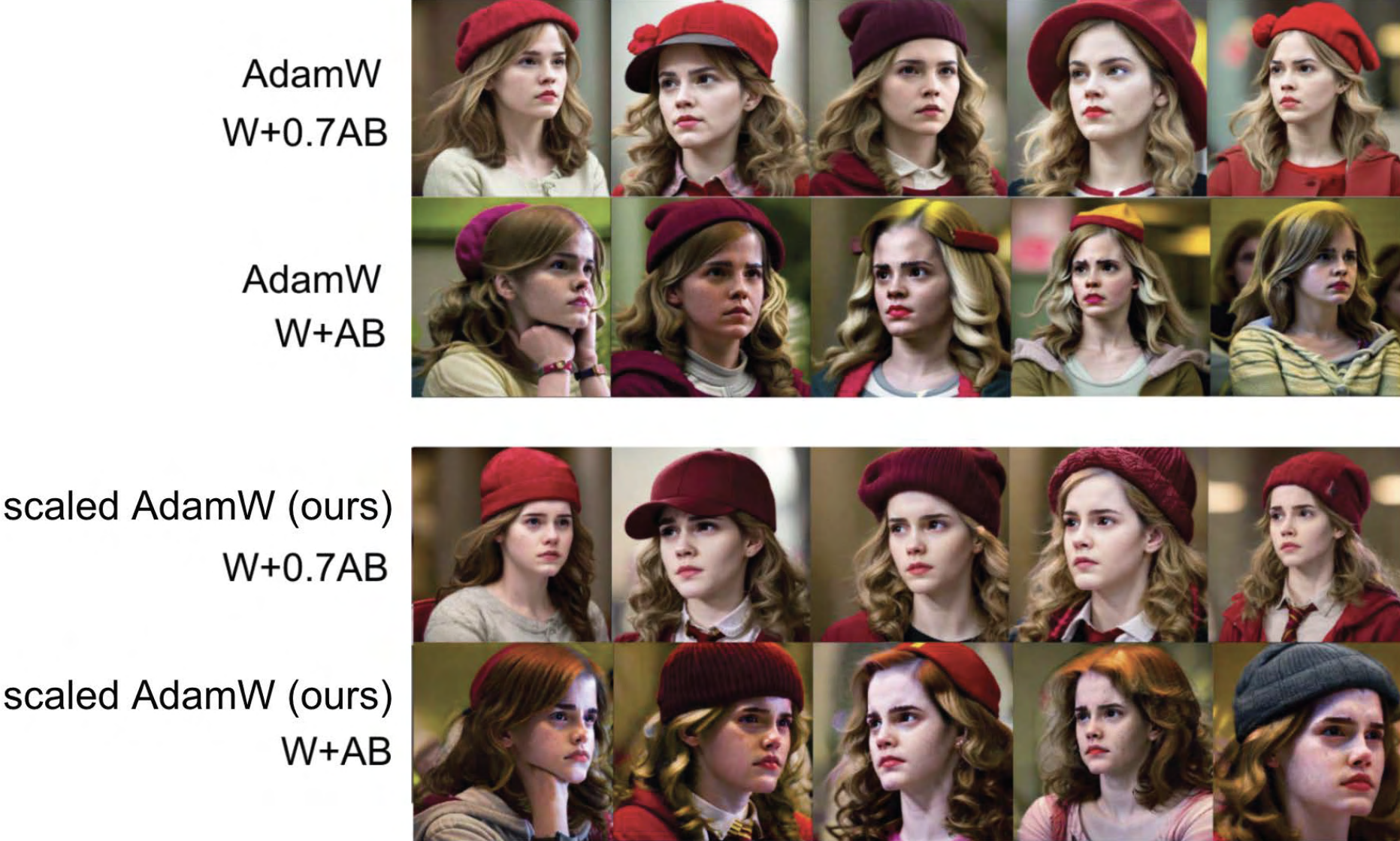}
\caption{Generation results for prompt 
 ``$\langle V_{\text{hermione}} \rangle$ wearing a red hat'' after fine-tuning on $15$ Hermione images. The above two rows are from AdamW optimizer and the bottom two rows are from our scaled AdamW optimizer. The first and third rows correspond to LoRA parameter fusion coefficient $0.7$ and the second and fourth rows correspond to LoRA parameter fusion coefficient $1.0$. Our scaled AdamW method generates images of higher quality compared to AdamW. See Appendix \ref{gen_quality} for experimental details.}\label{hermoine_1}
\end{figure*}

\begin{figure*}[h]
 \centering
\includegraphics[width=0.7\linewidth]{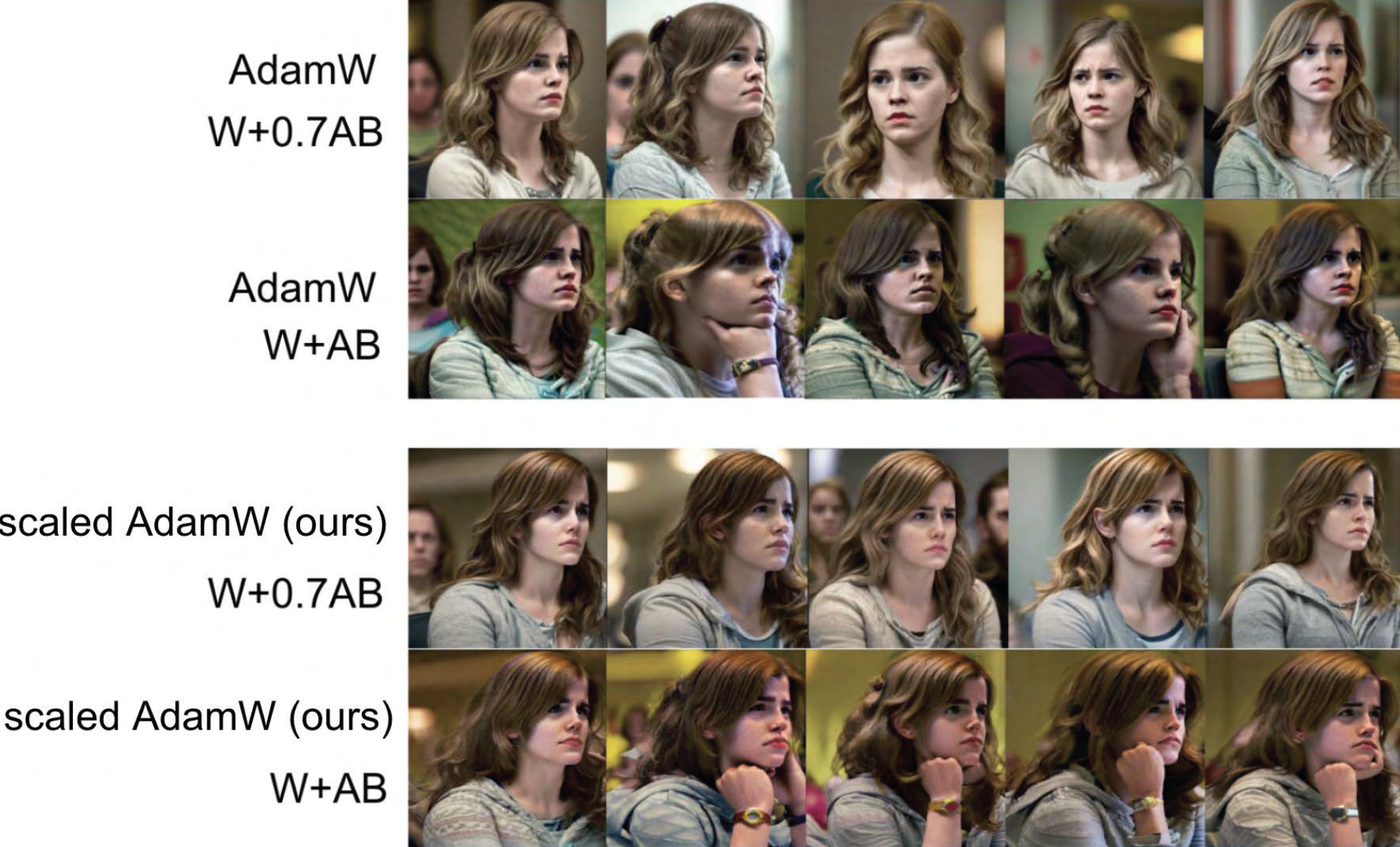}
\caption{Generation results for prompt ``$\langle V_{\text{hermione}} \rangle$'' after fine-tuning on $15$ Hermione images. The above two rows are from AdamW optimizer and the bottom two rows are from our scaled AdamW optimizer. The first and third rows correspond to LoRA parameter fusion coefficient $0.7$ and the second and fourth rows correspond to LoRA parameter fusion coefficient $1.0$. Our scaled AdamW method generates images of higher quality compared to AdamW. See Appendix \ref{gen_quality} for experimental details.}\label{hermoine_2}
\end{figure*}
\subsubsection{Varying Step Sizes}\label{lr_compare}
Here we test with varying step sizes. Note the Mix-of-Show repository uses AdamW as default optimizer with learning rate $1e-5$ for text-encoder tuning and $1e-4$ for U-Net tuning. SGD method is not used in original repository. We emprically observe that SGD requires larger learning rate compared to AdamW to generate sensible images. For this experiment, we test with three groups of learning rates, with ``Large" corresponds to $3e-1$ for SGD-type methods and $5e-4$ for AdamW-type methods; ``Medium" corresponds to $1e-1$ for SGD-type methods and $1e-4$ for AdamW-type methods; ``Small" corresponds to $5e-2$ for SGD-type methods and $1e-5$ for AdamW-type methods. We don't differentiate learning rates for U-Net and text-encoder tuning and the same learning rate is used for both. Note the default learning rate for AdamW falls between the ``Medium" learning rate and the ``Small" learning rate thus our learning rate choices are not random. Figure \ref{pencil_lr}, \ref{sit_lr} and \ref{diffusion_main2} show the generation results for three different prompts. It can be observed that our scaled optimizers are more robust to learning rate changes.
\begin{figure*}[h]
 \centering
\includegraphics[width=0.9\linewidth]{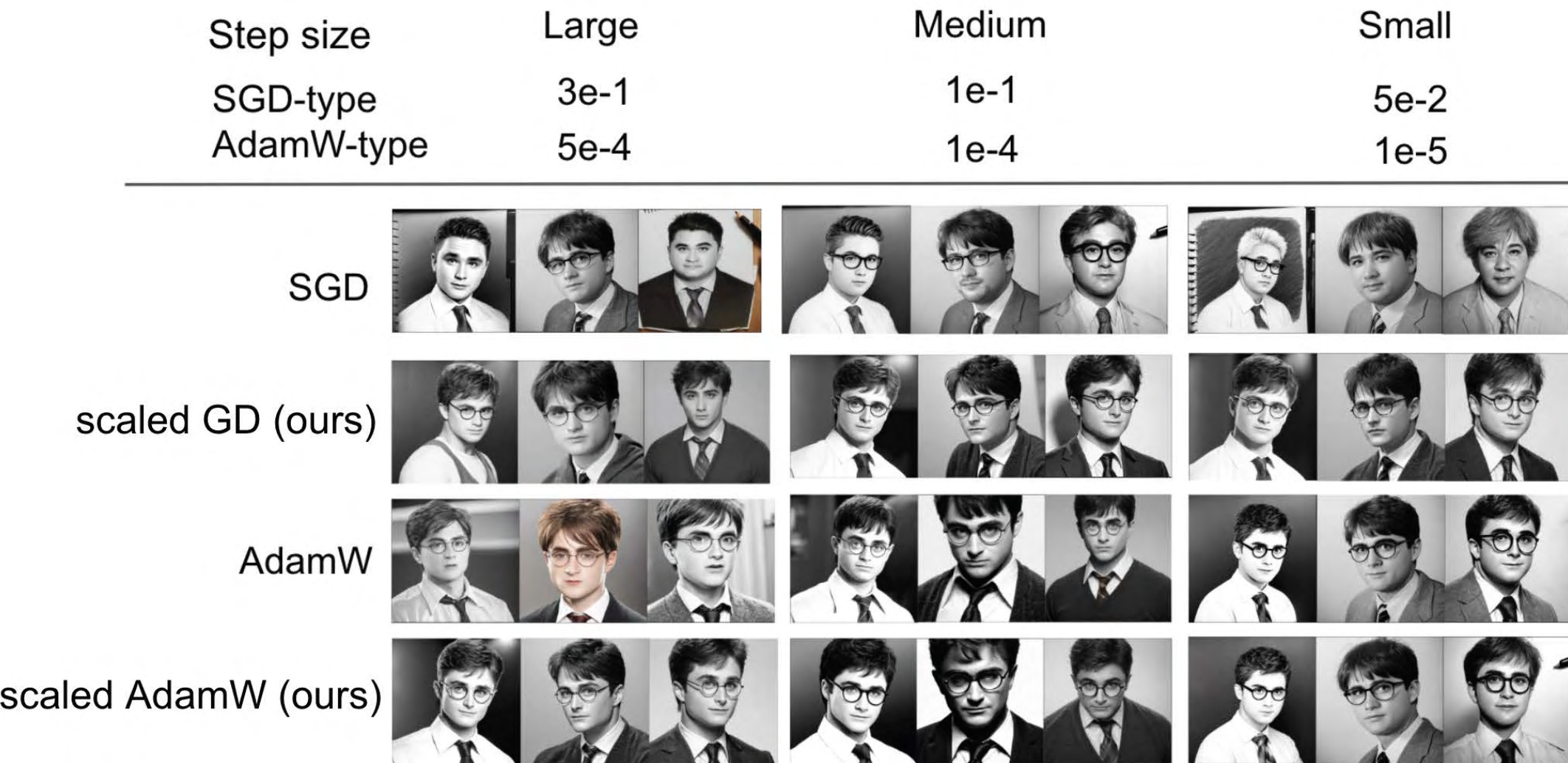}
\caption{Generation results for prompt ``a pencil sketch of $\langle V_{\text{potter}} \rangle$'' with different optimizers and different learning rates. See Appendix \ref{lr_compare} for experimental details. Default optimizer is AdamW with default learning rate $1e-4$ for U-Net tuning and $1e-5$ for text-encoder tuning. Our scaled optimizers generate better quality images and are more robust to learning rate changes.}\label{pencil_lr}
\end{figure*}

\begin{figure*}[h]
 \centering
\includegraphics[width=0.9\linewidth]{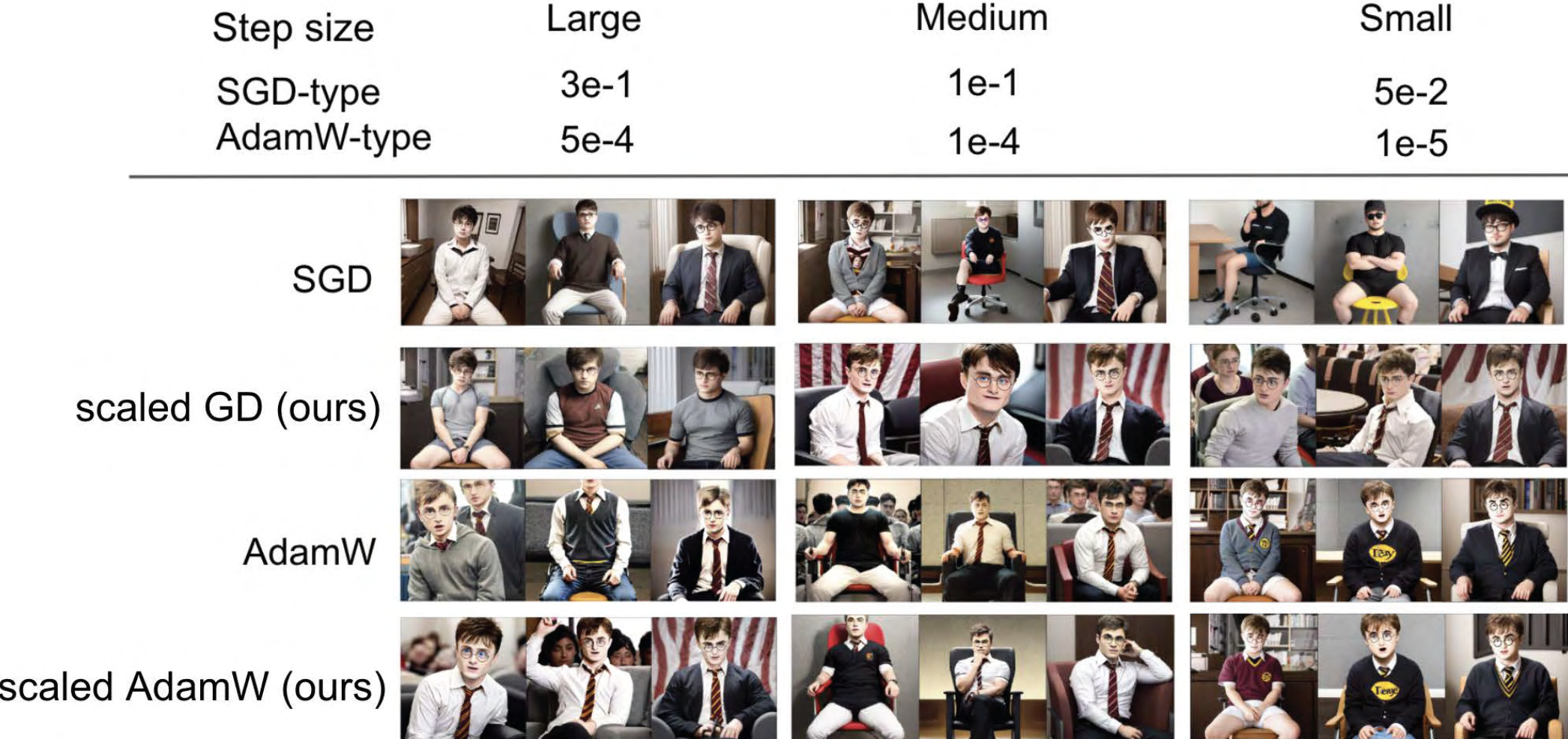}
\caption{Generation results for prompt ``$\langle V_{\text{potter}} \rangle$ sit on the chair'' with different optimizers and different learning rates. See Appendix \ref{lr_compare} for experimental details. Default optimizer is AdamW with default learning rate $1e-4$ for U-Net tuning and $1e-5$ for text-encoder tuning. Our scaled optimizers  generate better quality images and are more robust to learning rate changes.}\label{sit_lr}
\end{figure*}

 \begin{figure*}[h]
 \centering
\includegraphics[width=0.9\linewidth]{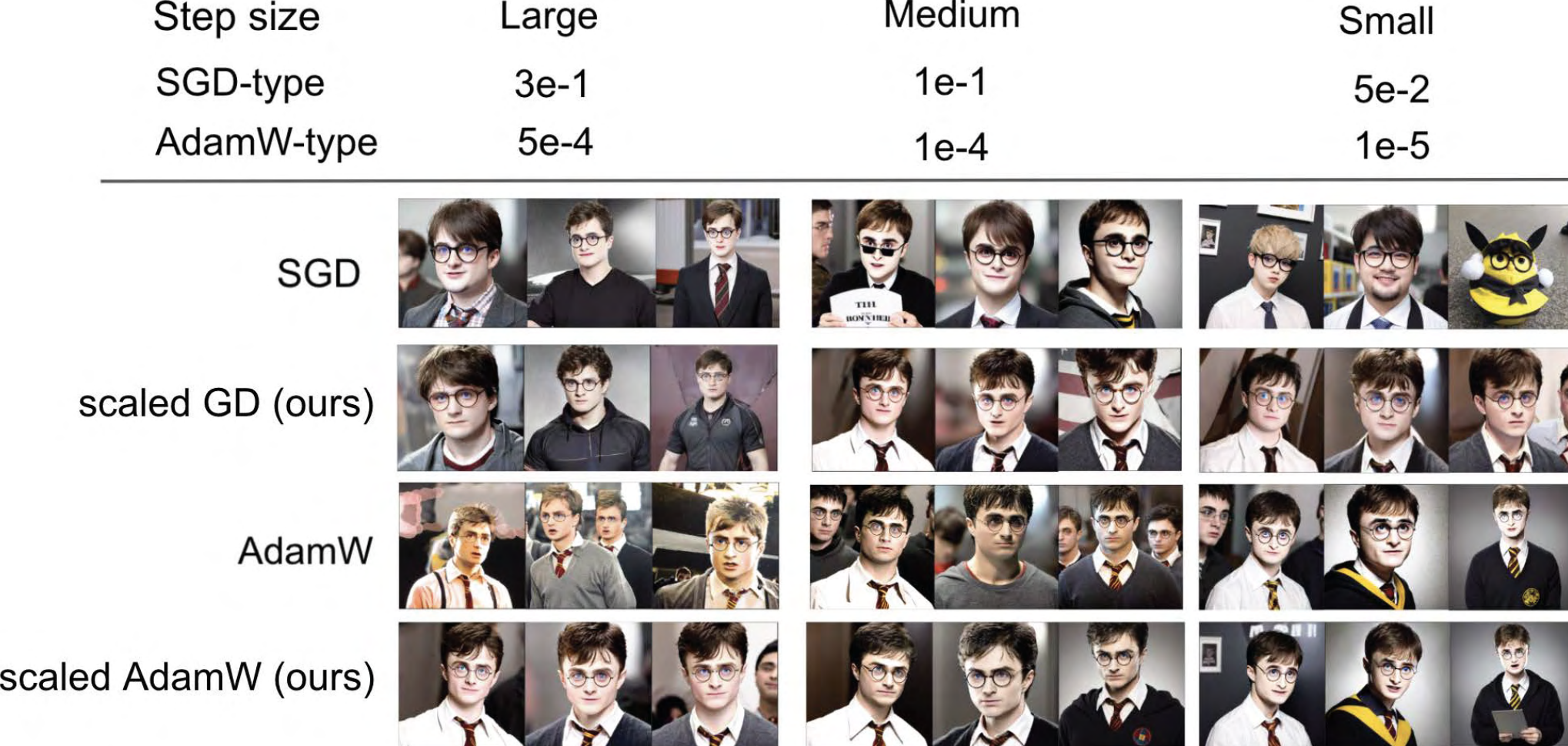}
\caption{Generation results for prompt ``a photo of $\langle V_{\text{potter}} \rangle$" with different optimizers and different learning rates. See Appendix \ref{lr_compare} for experimental details. Default optimizer is AdamW with default learning rate $1e-4$ for U-Net tuning and $1e-5$ for text-encoder tuning. Our scaled optimizers  generate better quality images and are more robust to learning rate changes.}\label{diffusion_main2}
\end{figure*}



\end{document}